\documentclass{article}

\usepackage{microtype}
\usepackage{graphicx}
\usepackage{subfigure}
\usepackage{booktabs} 
\usepackage{hhline}
\usepackage{hyperref}

\usepackage{tikz}
\usepackage{bm}
\usepackage{amsmath}
\usepackage{amsthm}
\usepackage{amsfonts}
\usepackage{amssymb}
\definecolor{teal}{rgb}{0,0.5,0.5}

\usepackage{url}
\usepackage{xspace}

\newenvironment{tightcenter}{%
  \setlength\topsep{0pt}
  \setlength\parskip{0pt}
  \begin{center}
}{%
  \end{center}
}

\usepackage[export]{adjustbox}
\usepackage[inline]{enumitem}

\newif\ifcomments
\commentstrue

\ifcomments

\newcommand{\cut}[1]{}
\newcommand{\maybecut}[1]{\textcolor{orange}{#1}}
\newcommand{\debating}[1]{\textcolor{red}{#1}}
\newcommand{\added}[1]{\textcolor{red}{#1}}
\newcommand{\donepromised}[1]{{}}
\newcommand{\toresolve}[1]{\textcolor{red}{#1}}
\newcommand{\remove}[1]{\textcolor{green}{#1}}
\newcommand{\removehide}[1]{}
\newcommand{\alt}[1]{\textcolor{brown}{#1}}
\newcommand{\commentsm}[1]{\textcolor{red}{({\bf SM:} #1)}}
\newcommand{\commentrti}[1]{{\color{teal}({\bf RTI:} #1)}}
\newcommand{\commentpv}[1]{\textcolor{blue}{({\bf PV:} #1)}}
\newcommand{\commental}[1]{\textcolor{magenta}{({\bf AL:} #1)}}

\newcommand{\old}[1]{\textcolor{red}{\sout{#1}}}
\newcommand{\new}[1]{\textcolor{BrickRed}{\uline{#1}}}
\else
    \newcommand{\commentsm}[1]{}
    \newcommand{\commentrti}[1]{}
    \newcommand{\commentpv}[1]{}
    \newcommand{\commental}[1]{}
    \newcommand{\maybecut}[1]{}
    \newcommand{\debating}[1]{}
    \newcommand{\added}[1]{\textcolor{red}{#1}}
    \newcommand{\donepromised}[1]{}
    \newcommand{\toresolve}[1]{}
    \newcommand{\remove}[1]{}
    \newcommand{\removehide}[1]{}
    \newcommand{\alt}[1]{}
    
    \newcommand{\old}[1]{}
    \newcommand{\new}[1]{}

\fi

\usepackage{mathtools}
\usepackage{amssymb}
\usepackage{amsmath}
\usepackage{xspace}

\newcommand{\expect}{\mathbb{E}}
\newcommand{\prob}{\mathbb{P}}
\newcommand{\reals}{\mathbb{R}}
\newcommand{\tuple}[1]{\langle{#1}\rangle}

\newcommand{\ltlalways}{\ensuremath{\Box}\xspace} 
\newcommand{\ltleventually}{\ensuremath{\Diamond}\xspace}
\newcommand{\ltlnext}{\ensuremath{\bigcirc}\xspace}

\newcommand{\true}{\ensuremath{\mathsf{true}}\xspace{}}
\newcommand{\false}{\ensuremath{\mathsf{false}}\xspace{}}
\newcommand{\ltluntil}{\ensuremath{\operatorname{\mathsf{U}}}\xspace{}}

\newcommand{\textltluntil}{\ensuremath{\operatorname{\mathsf{U}}}\xspace}

\newcommand{\mprog}{\operatorname{prog}}

\newcommand{\gnnpp}{$\text{GNN}_{\operatorname{prog}}^{\operatorname{pre}}$}
\newcommand{\gnnp}{$\text{GNN}_{\operatorname{prog}}$}

\newcommand{\grupp}{$\text{GRU}_{\operatorname{prog}}^{\operatorname{pre}}$}
\newcommand{\grup}{$\text{GRU}_{\operatorname{prog}}$}

\newcommand{\lstmp}{$\text{LSTM}_{\operatorname{prog}}$}

\newcommand{\simLtl}{\texttt{LTLBootcamp}\xspace}
\newcommand{\letterWorld}{\texttt{LetterWorld}\xspace}
\newcommand{\zoneEnv}{\texttt{ZoneEnv}\xspace}
\usepackage{multirow}

\newtheorem{theorem}{Theorem}[section]

\theoremstyle{definition}
\newtheorem{definition}{Definition}[section]

\usepackage{array}
\usepackage{booktabs}
\usepackage{rotating}

\newenvironment{itpars}
  {\par\itshape}
  {\par}

\usepackage[accepted]{icml2021}

\icmltitlerunning{LTL2Action: Generalizing LTL Instructions for Multi-Task RL}

\begin{document}

\twocolumn[
\icmltitle{LTL2Action: Generalizing LTL Instructions for Multi-Task RL}



\icmlsetsymbol{equal}{*}

\begin{icmlauthorlist}
\icmlauthor{Pashootan Vaezipoor}{equal,to,vector}
\icmlauthor{Andrew C. Li}{equal,to,vector}
\icmlauthor{Rodrigo Toro Icarte}{to,vector}
\icmlauthor{Sheila McIlraith}{to,vector,sri}
\end{icmlauthorlist}

\icmlaffiliation{to}{Department of Computer Science, University of Toronto}
\icmlaffiliation{vector}{Vector Institute for Artificial Intelligence}
\icmlaffiliation{sri}{Schwartz Reisman Institute for Technology and Society}

\icmlcorrespondingauthor{Pashootan Vaezipoor}{pashootan@cs.toronto.edu}
\icmlcorrespondingauthor{Andrew Li}{andrewli@cs.toronto.edu}

\icmlkeywords{Machine Learning, ICML}

\vskip 0.3in
]



\printAffiliationsAndNotice{\icmlEqualContribution} 

\begin{abstract}

We address the problem of teaching a deep reinforcement learning (RL) agent to follow instructions in multi-task environments. Instructions are expressed in a well-known formal language -- \emph{linear temporal logic} (LTL) -- and can specify a diversity of complex, temporally extended behaviours, including conditionals and alternative realizations.
Our proposed learning approach exploits the compositional syntax and the semantics of LTL, enabling our RL agent to learn task-conditioned policies that generalize to new instructions, not observed during training. 
To reduce the overhead of learning LTL semantics, we introduce an environment-agnostic LTL pretraining scheme which improves sample-efficiency in downstream environments. 
Experiments on discrete and continuous domains target combinatorial task sets of up to $\sim10^{39}$ unique tasks and demonstrate the strength of our approach in learning to solve (unseen) tasks, given LTL instructions. 

\removehide
{
We introduce a novel \toresolve{neurosymbolic} approach for teaching a deep reinforcement learning (RL) agent a wide array of temporally-extended, taskable behaviors. Complex tasks are formally specified to the agent through linear temporal logic (LTL), while a well-defined method called LTL-progression is used to efficiently learn non-Markovian, task-conditioned policies. While decomposing complex tasks into independent subtasks is a ubiquitous technique in the multitask RL literature, we demonstrate that this can often lead to sub-optimal behaviors. Instead, we propose a holistic approach that directly encodes the full task using neural networks and explore the impact of different neural architectures on generalization to unseen tasks. To circumvent the overhead of learning LTL semantics, we introduce an environment-agnostic LTL pretraining scheme which is able to improve sample-efficiency in downstream environments. Experiments on discrete and continuous RL domains demonstrate the merits of our approach in learning to solve complex unseen LTL tasks.
}


\end{abstract}
\section{Introduction}

%
%
%


A long-standing aspiration of artificial intelligence is to build agents that can understand and follow human instructions to solve problems \cite{mccarthy1960programs}. Recent advances in deep 
and reinforcement learning (RL) have made it possible to learn a policy that decides the next action conditioned on the current observation and a natural language instruction. Given enough training data, the learned policy will show some degree of generalization to unseen instructions 
\citep[e.g.,][]{DBLP:journals/corr/HermannHGWFSSCJ17,DBLP:conf/icml/OhSLK17,DBLP:conf/aaai/ChaplotSPRS18,DBLP:conf/iclr/YuZX18,DBLP:conf/iclr/Co-ReyesGSAADAL19,DBLP:conf/nips/JiangGMF19,DBLP:conf/ijcai/LuketinaNFFAGWR19}. Unfortunately, such approaches do not scale well because they require (for every possible environment) manually building a large training set comprised of natural language instructions with their corresponding reward functions.


Motivated by this observation, recent works have explored using structured or formal languages (instead of natural language) to instruct RL agents.
Such languages offer several desirable properties for RL, including unambiguous semantics, and compact compositional syntax that enables RL practitioners to (automatically) generate massive training data to teach RL agents to follow instructions. Examples of such languages include  policy sketches \cite{andreas2017modular}, task graphs \cite{DBLP:conf/nips/SohnOL18}, procedural programs \cite{DBLP:conf/iclr/SunWL20}, declarative programs \cite{DBLP:journals/corr/DenilCCSF17}, reward machines \cite{DBLP:journals/corr/abs-2010-03950}, and temporal logic  \cite{leon2020systematic}. 
Many of these methods exploit compositional syntax to decompose instructions into smaller subtasks that are solved independently, without consideration for the subtasks that follow. This can lead to subtask policies that are individually optimal, but when combined are suboptimal with respect to the instructions as a whole.
We refer to these as \emph{myopic approaches}.

%

\removehide{For instance, two tasks 
``\textit{get-coffee and then get-mail}''`and ``\textit{get-coffee and then meet-boss}'' will reuse the same policy to achieve \textit{get-coffee}. Unfortunately, doing so might lead to suboptimal behaviour because the optimal way to achieve \textit{get-coffee} depends on what the agent has to do next (e.g., the optimal coffee machine to realize \textit{get-coffee} depends on where the agent has to go afterwards). We call these approaches \emph{myopic} as they 
optimize the completion of subtasks without considering the rest of the task.
}

In this work, we use \emph{linear temporal logic (LTL)} \cite{DBLP:conf/focs/Pnueli77} over a domain-specific 
vocabulary (e.g., \textit{have-coffee}) to instruct RL agents to learn 
policies that generalize well to unseen instructions without compromising optimality guarantees -- in contrast to typical myopic methods.
LTL is an expressive formal language that combines temporal modalities such as \emph{eventually}, \emph{until}, and \emph{always} with binary predicates that establish the truth or falsity of an event or property (e.g., \textit{have-coffee}), composed via logical connectives to support 
specification of goal sequences, partial-order tasks, safety constraints, and much more.
Our learning algorithm exploits a semantics-preserving rewriting operation, called \emph{LTL progression}, that allows the agent to identify aspects of the original instructions that remain to be addressed in the context of an evolving experience.
This enables learning policies in a non-myopic manner, all the while preserving optimality guarantees and supporting generalization.
%

%

Our approach is realized in a deep RL setting, exploiting event detectors to recognize domain vocabulary. \cut{The compositional syntax of LTL allows us to randomly sample millions of LTL instructions, systematically transforming them into reward functions that reward completion of the instructions.} We encode LTL instructions 
using an LSTM, GRU, or a Graph Neural Network (GNN). To reduce the overhead of learning LTL semantics, we introduce an environment-agnostic LTL pretraining scheme. 
\cut{Our LTL-progression enhanced RL learns a policy given the current observation and LTL instruction.} We evaluate our approach on 
discrete and continuous domains. \removehide{ assessing the effectiveness of our LTL-progression based approach relative to a myopic approach; the relative merit of 
architectures that encode LTL using LSTM, GRU, or GNN; and the \debating{performance??} benefit of environment-agnostic pretraining to learn LTL semantics.}  
%
Our contributions are as follows:

\begin{itemize}[topsep=0pt,itemsep=2pt,partopsep=0pt, parsep=0pt,leftmargin=*]
    
    \item We propose a novel approach for teaching RL agents to follow LTL instructions that has theoretical advantages over existing RL methods employing LTL instructions \cite{kuo2020encoding,leon2020systematic}. This leads to better generalization performance in our experiments. 
    
    
    \item We show that encoding LTL instructions via a neural architecture equipped with LTL progression yielded higher reward policies relative to a myopic approach. Out of the neural architectures GNN offered better generalization compared to LSTM and GRU.
   
    \item Lastly, we demonstrate that applying an environment-agnostic LTL pretraining scheme improves sample efficiency on downstream tasks.
    \removehide{
    \item exploits the compositional structure of the language to learn subtasks, while \remove{preserving consideration of the instructions in their entirety, and thus} preserving optimality guarantees lost by myopic approaches; 
    \item \alt{benefits from the efficiencies afforded by subtask decomposition while learning a memoryless policy and preserving optimality guarantees lost by myopic approaches.}
    \item yields optimal stationary policies with the same expected discounted return as the optimal policy to realize the non-Markovian LTL instructions, conditioned on the state-action history;
    \item characterizes policy learning with temporally extended instructions as an MDP in a transformed space thereby yielding theoretical advantage relative to RNN based methods that are characterized as POMDPs  \cite{kuo2020encoding,leon2020systematic}.  
    \item\commentsm{other comments we might make somewhere - works w/ any RL alg, mention DEEP Rl again?}
    }
    \end{itemize}

\removehide{
\noindent Experimental results show that:
\begin{itemize}[topsep=0pt,itemsep=2pt,partopsep=0pt, parsep=0pt,leftmargin=*]
\item our approach has superior performance over competing methods which are myopic or do not use LTL progression.
\commental{I don't like non-stationary.} \alt{competing approaches or approaches that don't use LTL progression}
\item environment-agnostic LTL pretraining improves sample efficiency. 
\item graph neural networks offer better generalization than sequence models for encoding LTL instructions. 
\end{itemize}
}

\removehide{
In more details, our approach requires a predefined set of \textit{event detectors}. These detectors indicate when different subtasks are completed in the environment. Then, we use LTL to compose occurrences of those detectors over time to define complex tasks. Given the detectors, we can randomly sample millions of LTL tasks for the agent to learn from. We then show how to transforms LTL tasks into their corresponding reward functions by rewarding the agent when the task is completed. Finally, we use deep RL to learn a policy that considers the current observation and LTL instructions when selecting actions. Crucially, we progress the LTL instruction as the agent interacts with the environment. Results on discrete and continuous domains show that progressing the LTL task leads to better generalization performance than following a myopic approach. In addition, we show that the LTL tasks can be encoded using an LSTM, GRU, or a graph neural network (GNN), and that all those encoders work well in practice. That said, our best results were obtained by encoding the LTL task using a GNN. Finally, we show that the LTL encoder can be pretrained in an environment-agnostic fashion and that doing so results on better sample
efficiency.
%
With that, the main contributions of our paper can be summarized as follows:
\begin{itemize}
    \setlength\itemsep{0em}
    \item A novel approach for teaching RL agents to follow LTL instructions. Our approach has theoretical advantages over existing methods to instruct RL agents using LTL  \cite{kuo2020encoding,leon2020systematic}. This derives on a better generalization performance in our experiments. 
    \item Empirical results comparing different methods to encode LTL instructions (LSTM, GRU, GNN).
    \item To show that the LTL encoder can be pretrained in an environment-agnostic manner and the empirical advantages of doing so.
\end{itemize}
}

\section{Reinforcement Learning}

RL agents learn optimal behaviours by interacting with an environment. Usually, the environment is modelled as a \emph{Markov Decision Process (MDP)}. An MDP is a tuple $\mathcal{M} = \tuple{S,T,A,\prob,R,\gamma,\mu}$, where $S$ is a finite set of \emph{states}, $T \subseteq S$ is a finite set of \emph{terminal states}, $A$ is a finite set of \emph{actions}, $\prob(s'|s,a)$ is the \emph{transition probability distribution}, $R:S \times A\times S \rightarrow \reals$ is the \emph{reward function}, $\gamma$ is the \emph{discount factor}, and $\mu$ is the \emph{initial state distribution}. 

The interactions with the environment are divided into \emph{episodes}. At the beginning of an episode, the environment is set at some initial state $s_0\in S$ sampled from $\mu$. Then, at time step $t$, the agent observes the current state $s_t \in S$ and executes an action $a_t \in A$ according to some \emph{policy} $\pi(a_t|s_t)$ -- which is a probability distribution from states to actions. In response, the environment returns the next state $s_{t+1}$ sampled from $\prob(s_{t+1}|s_t,a_t)$ and an immediate reward $r_t = R(s_t,a_t,s_{t+1})$. This process then repeats until reaching a terminal state (starting a new episode).
The agent's objective is to learn an \emph{optimal policy} $\pi^*(a|s)$ that maximizes the expected discounted return $\expect_{\pi}{\left[\sum_{k=0}^{\infty}\gamma^kr_{t+k} \middle| S_t=s\right]}$ when starting from any state $s \in S$ and time step $t$.


\section{Multitask Learning with LTL}
\label{sec:multitask-ltl}


In order to instruct RL agents using language, the first step is to agree upon a common 
vocabulary between us and the agent. In this work, we use a finite set of propositional symbols $\mathcal{P}$ as the vocabulary, representing high-level events or properties (henceforth ``events'') whose occurrences in the environment can be detected by the agent. For instance, in a smart home environment, $\mathcal{P}$ could include events such as opening the living room window, activating the fan, turning on/off the stove, or entering the living room. Then, we use LTL to compose temporally-extended occurrences of these events into instructions. For example, two possible instructions that can be expressed in LTL (but described here in plain English) are (1) ``Open the living room window and activate the fan in any order, then turn on the stove" and (2) ``Open the living room window but don't enter the living room until the stove is turned off". 

In this section, we discuss how to specify instructions using LTL and automatically transform those instructions into reward functions, we formally define the RL problem of learning a policy that generalizes to unseen LTL instructions.

%

\subsection{Linear Temporal Logic (LTL)}

LTL extends propositional logic with two temporal operators: $\ltlnext$ (\emph{next}) and $\textltluntil$ (\emph{until}). Given a finite set of propositional symbols $\mathcal{P}$, the syntax of an LTL formula is defined as follows \cite{DBLP:books/daglib/0020348}:
\begin{equation*}
    \varphi \Coloneqq p \;|\; \neg \varphi \;|\; \varphi \wedge \psi \;|\; \ltlnext \varphi \;|\; \varphi \ltluntil \psi \text{ ~~~~where $p \in \mathcal{P}$}
\end{equation*}
In contrast to propositional logic, LTL formulas are evaluated over sequences of observations (i.e., \emph{truth assignments} to the propositional symbols in $\mathcal{P}$). Intuitively, the formula $\ltlnext \varphi$ (\textit{next} $\varphi$) holds if $\varphi$ holds at the next time step and $\varphi \ltluntil \psi$ ($\varphi$ \textit{until} $\psi$) holds if $\varphi$ holds until $\psi$ holds. 

Formally, the truth value of an LTL formula is determined relative to an infinite sequence of truth assignments $\sigma=\langle \sigma_0, \sigma_1, \sigma_2, \ldots \rangle$ for $\mathcal{P}$, where $p \in \sigma_i$ iff proposition $p \in \mathcal{P}$ holds at time step $i$. Then, $\sigma$ \emph{satisfies} $\varphi$ at time $i \geq 0$, denoted by $\tuple{\sigma,i}\models\varphi$, as follows:
\begin{itemize}[topsep=0pt,itemsep=0pt,partopsep=0pt, parsep=0pt]
    \item $\tuple{\sigma,i}\models p$ iff $p\in \sigma_i$, where $p \in \mathcal{P}$
    \item $\tuple{\sigma,i}\models \neg \varphi$ iff $\tuple{\sigma,i}\not\models\varphi$
    \item $\tuple{\sigma,i}\models (\varphi\wedge\psi)$ iff $\tuple{\sigma,i}\models\varphi$ and $\tuple{\sigma,i}\models\psi$
    \item $\tuple{\sigma,i}\models\ltlnext\varphi$ iff $\tuple{\sigma,i+1}\models\varphi$
    \item $\tuple{\sigma,i}\models\varphi\ltluntil\psi$ iff there exists $j$ such that $i\le j$ 
    and\\\verb!     !$\tuple{\sigma,j}\models\psi$, and $\tuple{\sigma,k}\models\varphi$ for all $k \in [i,j)$
\end{itemize}
A sequence $\sigma$ is then said to \emph{satisfy} $\varphi$ iff $\tuple{\sigma,0}\models\varphi$.


Any LTL formula can be define in terms of $p \in \mathcal{P}$, $\wedge$ (\emph{and}), $\neg$ (\emph{negation}), $\ltlnext$ (\emph{next}), and $\textltluntil$ (\emph{until}). However, from these operators, we can also define the Boolean operators $\vee$ (\emph{or}) and $\rightarrow$ (\emph{implication}), and the temporal operators $\ltlalways$ (\emph{always}) and $\ltleventually$ (\emph{eventually}), where $\tuple{\sigma,0}\models\ltlalways \varphi$ if $\varphi$ always holds in $\sigma$, and $\tuple{\sigma,0}\models\ltleventually \varphi$ if $\varphi$ holds at some point in $\varphi$.

As an illustrative example, consider the MiniGrid \cite{gym_minigrid} environment in Figure~\ref{fig:toy}. There are two rooms -- one with a blue and a red square, and one with a blue and a green square. The agent, represented by a red triangle, can rotate left and right, and move forward. Let's say that the set of propositions $\mathcal{P}$ includes $\mathrm{R}$, $\mathrm{G}$, and $\mathrm{B}$, which are $\true$ if and only if the agent is standing on a red/green/blue square (respectively) in the current time step. Then, we can define a wide variety of tasks using LTL:
\begin{itemize}[topsep=0pt,itemsep=1pt,partopsep=0pt, parsep=0pt]
    \item Single goal: $\ltleventually \mathrm{R}$ (reach a red square).
    \item Goal sequences: $\ltleventually (\mathrm{R} \wedge \ltleventually\mathrm{G})$ (reach red and then green).
    \item Disjunctive goals: $\ltleventually \mathrm{R} \vee \ltleventually\mathrm{G}$ (reach red or green).
    \item Conjunctive goals: $\ltleventually \mathrm{R} \wedge \ltleventually\mathrm{G}$ (reach red and green\footnote{In any order.}).
    \item Safety constraints: $\ltlalways \neg \mathrm{B}$ (do not touch a blue square).
\end{itemize}
%
%
We can also combine these tasks to define new tasks, e.g., ``\textit{go to a red square and then a green square but do not touch a blue square}" can be expressed as $\ltleventually (\mathrm{R} \wedge \ltleventually\mathrm{G}) \wedge \ltlalways \neg \mathrm{B}$.


While LTL is interpreted over infinite sequences, the truth of many LTL formulas can be ensured after a finite number of steps. For instance, the formula $\ltleventually \mathrm{R}$ (\textit{eventually red}) is satisfied by any infinite sequence where $\mathrm{R}$ is true at some point. Hence, as soon as $\mathrm{R}$ holds in a finite sequence, we know that $\ltleventually \mathrm{R}$ will hold. 
Similarly, a formula such as $\ltlalways \neg \mathrm{B}$ is immediately determined to be unsatisfied by an occurrence of $\mathrm{B}$, regardless of what follows.
\subsection{From LTL Instructions to Rewards}
\label{sec:ltl_tasks}

So far, our discussion about LTL instructions has been environment-agnostic (the syntax and semantics of LTL are independent of the environment). Now, we show how to reward an RL agent for realizing LTL instructions via an MDP. Following previous works \cite{DBLP:conf/atal/IcarteKVM18, alur2019}, we accomplish this by using a \emph{labelling function} $L: S \times A \rightarrow 2^\mathcal{P}$. The labelling function $L(s,a)$ assigns truth values to the propositions in $\mathcal{P}$ given the current state $s \in S$ of the environment and the action $a \in A$ selected by the agent. 
One may think of the labelling function as having a collection of \emph{event detectors} that fire when the propositions in $\mathcal{P}$ hold in the environment.
In our running example, $\mathrm{R} \in L(s,a)$ iff the agent is on top of the red square and similarly for G (green) and B (blue).

Given a labelling function, the agent can automatically evaluate whether an LTL instruction has been satisfied or falsified. If the instruction is satisfied (i.e., completed) we give the agent a reward of 1 and if the instruction is falsified (e.g., the agent breaks a safety constraint) we penalize the agent with a reward of $-1$. The episode ends as soon as the instruction is satisfied or falsified. Formally, given an LTL instruction $\varphi$ over $\mathcal{P}$ and a labelling function $L: S \times A \rightarrow 2^\mathcal{P}$ and the sequence of states and actions seen so far in the episode: $s_1, a_1, ..., s_t, a_t$, the reward function is defined as follows:
\begin{equation}
 R_\varphi(s_1, a_1, ..., s_t, a_t) = 
\begin{cases}
    1 &\text{if } \sigma_1...\sigma_t \models \varphi \\
    -1 &\text{if } \sigma_1...\sigma_t \models \neg \varphi\\
    0 &\text{otherwise }
\end{cases} \ ,
\label{eq:r_varphi}
\end{equation}
where $\sigma_i = L(s_i,a_i)$. 

Observe that the reward function specified above renders a non-zero reward if the LTL formula can be determined to be satisfied or unsatisfied in a finite number of steps. This is guaranteed to be the case for various fragments of LTL, including co-safe LTL \cite{kupferman2001model} and for so-called LTL-f (the variant of LTL that is interpreted over finite traces). 
For LTL formulas that cannot be verified or falsified in finite time (e.g. $\ltlalways \ltleventually G$), the agent receives no meaningful reward signal. 
One way to address such LTL formulas is to alter the reward function to render an appropriate reward after a very large but finite number of steps (e.g., $10^6$ steps), with commensurate guarantees regarding the resulting policies. 
The topic of an appropriate reward function for general LTL formulas is addressed in \cite{hasanbeig2018logically} and explored in \cite{littman2017environment}.

Finally, note that this reward function might be non-Markovian, as it depends on sequences of states and actions, making the overall learning problem partially observable. We discuss how to deal with this issue below.


\subsection{Instructing RL Agents using LTL}
\label{sec:progression}
 We now formalize the problem of learning a policy that can follow LTL instructions.\footnote{Hereafter, LTL instruction/task may be used interchangeably.} 
Given an MDP without a reward function $\mathcal{M}_e = \tuple{S,T,A,\prob,\gamma,\mu}$, a finite set of propositional symbols $\mathcal{P}$, a labelling function $L: S \times A \rightarrow 2^\mathcal{P}$, a finite (but potentially large) set of LTL formulas $\Phi$, and a probability distribution $\tau$ over those formulas $\varphi \in \Phi$, our goal is to learn an optimal policy $\pi^*(a_t|s_1, a_1, ..., s_t,\varphi)$  w.r.t.\ $R_\varphi(s_1, a_1, ..., s_t, a_t)$ for all $\varphi \in \Phi$. To learn this policy, the agent will sample a new LTL task $\varphi$ from $\tau$ on every episode and, during that episode, it will be rewarded according to $R_\varphi$. The episode ends when the task is completed, falsified, or a terminal state is reached.

A major challenge to solving this problem is that the optimal policy $\pi^*(a_t|s_1, a_1, ..., s_t,\varphi)$ has to consider the whole history of states and actions since the reward function is non-Markovian. To handle this issue, \citet{kuo2020encoding} proposed to encode the policy using a recurrent neural network. However, here we show that we can overcome this complexity by exploiting a procedure known as LTL progression \cite{DBLP:journals/ai/BacchusK00}.

\begin{definition}
    Given an LTL formula $\varphi$ and a truth assignment $\sigma$ over $\mathcal{P}$, $\mprog(\sigma,\varphi)$ is defined as follows:
    \begin{itemize}[topsep=0pt,itemsep=0pt,partopsep=0pt, parsep=0pt]

        \item $\mprog(\sigma,p) = \true$ if $p \in \sigma$, where $p \in \mathcal{P}$
        \item $\mprog(\sigma,p) = \false$ if $p \not\in \sigma$, where $p \in \mathcal{P}$
        \item $\mprog(\sigma,\neg \varphi) = \neg \mprog(\sigma,\varphi)$
        \item $\mprog(\sigma,\varphi\wedge\psi) = \mprog(\sigma,\varphi) \wedge \mprog(\sigma,\psi)$
        \item $\mprog(\sigma,\ltlnext\varphi) = \varphi$
        \item $\mprog(\sigma,\varphi\ltluntil\psi) = \mprog(\sigma,\psi) \vee (\mprog(\sigma,\varphi) \wedge \varphi\ltluntil\psi)$    \end{itemize}
    \label{def:progression}
\end{definition} 
The $\mprog$ operator is a semantics-preserving rewriting procedure that takes an LTL formula and current labelled state as input and returns a formula that identifies aspects of the original instructions that remain to be addressed. Progress towards completion of the task is reflected in diminished remaining instructions.
For instance, the task $\ltleventually (\mathrm{R} \wedge \ltleventually\mathrm{G})$ (\textit{go to red and then to green}) will progress to $\ltleventually\mathrm{G}$ (\textit{go to green}) as soon as $\mathrm{R}$ holds in the environment. We use LTL progression to make the reward function $R_\varphi$ Markovian. We achieve this by (1) augmenting the MDP state with the current LTL task $\varphi$ that the agent is solving, (2) progressing $\varphi$ after each step given by the agent in the environment, and (3) rewarding the agent when $\varphi$ progresses to $\true$ ($+1$) or $\false$ ($-1$). This gives rise to an augmented MDP, that we call a \emph{Taskable MDP}, where the LTL instructions are part of the MDP states:

\begin{definition}[Taskable MDP]
    \label{def:taskableMDP}
    Given an MDP without a reward function $\mathcal{M}_e = \tuple{S,T,A,\prob,\gamma,\mu}$, a finite set of propositional symbols $\mathcal{P}$, a labelling function $L: S \times A \rightarrow 2^\mathcal{P}$, a finite set of LTL formulas $\Phi$, and a probability distribution $\tau$ over $\Phi$, we construct \mbox{\emph{Taskable MDP}} $\mathcal{M}_\Phi = \tuple{S',T',A,\prob',R',\gamma,\mu'}$, where $S' = S \times \operatorname{cl}(\Phi)$, $T' = \{\tuple{s,\varphi} \ |\ s \in T \text{ or } \varphi \in \{\true,\false\} \}$,  $\prob'(\tuple{s',\varphi'}|\tuple{s,\varphi},a)=\prob(s'|s,a)$ if $\varphi' = \mprog(L(s,a),\varphi)$ (zero otherwise), $\mu'(\tuple{s,\varphi}) = \mu(s) \cdot \tau(\varphi)$, and
    $$R'(\tuple{s,\varphi},a) = 
    \begin{cases}
        1 &\text{if } \mprog(L(s,a),\varphi) = \true \\
        -1 &\text{if } \mprog(L(s,a),\varphi) = \false\\
        0 &\text{otherwise }
    \end{cases} \ .
    $$
    Here $\operatorname{cl}(\Phi)$ denotes the \emph{progression closure} of $\Phi$, i.e., the smallest set containing $\Phi$ that is closed under progression. 
\end{definition}

With that, the main theorem of our paper shows that an optimal policy $\pi^*(a_t|s_1, a_1, ..., s_t,\varphi)$ to solve any LTL task $\varphi \in \Phi$ in some environment $\mathcal{M}_e$ achieves the same expected discounted return as
an optimal policy $\pi^*(a|s,\varphi)$ for the Taskable MDP $\mathcal{M}_\Phi$ constructed using $\mathcal{M}_e$ and $\Phi$ (we prove this theorem in Appendix~\ref{app:theorem}).
\begin{theorem}
    Let $\mathcal{M}_\Phi = \tuple{S',T',A,\prob',R',\gamma,\mu'}$ be a Taskable MDP constructed from an MDP without a reward function $\mathcal{M}_e = \tuple{S,T,A,\prob,\gamma,\mu}$, a finite set of propositional symbols $\mathcal{P}$, a labelling function $L: S \times A \rightarrow 2^\mathcal{P}$, a finite set of LTL formulas $\Phi$, and a probability distribution $\tau$ over $\Phi$. Then, an optimal stationary policy $\pi_\Phi^*(a|s,\varphi)$ for $\mathcal{M}_\Phi$ achieves the same expected discounted return as an optimal non-stationary policy $\pi_\varphi^*(a_t|s, a_1, ..., s_t,\varphi)$ for $\mathcal{M}_e$ w.r.t.\ $R_\varphi$, as defined in \eqref{eq:r_varphi}, for all $s \in S$ and $\varphi \in \Phi$. 
    \label{theo:cross-product}
\end{theorem}

\subsection{Discussion and Bibliographical Remarks}
\label{sec:ltl_discussion}

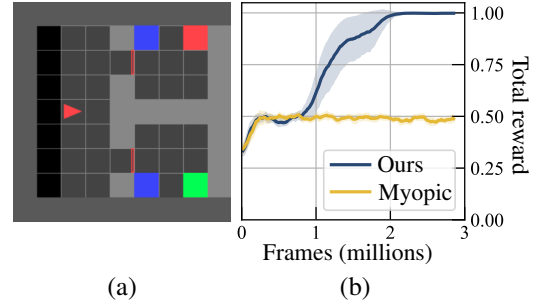
\begin{figure}[tb]
  \centering
  \small
  
  \begin{tikzpicture}
  \node at (-2,0.145) {\includegraphics[width=0.35\columnwidth,valign=t]{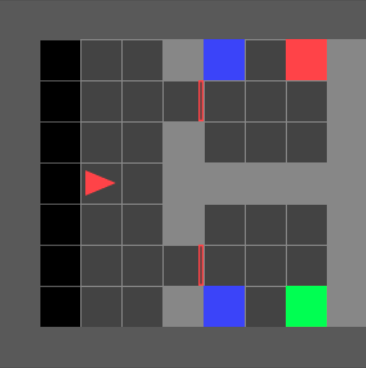}};
  \node at (1.3,0) {\includegraphics[width=0.43\columnwidth,valign=t]{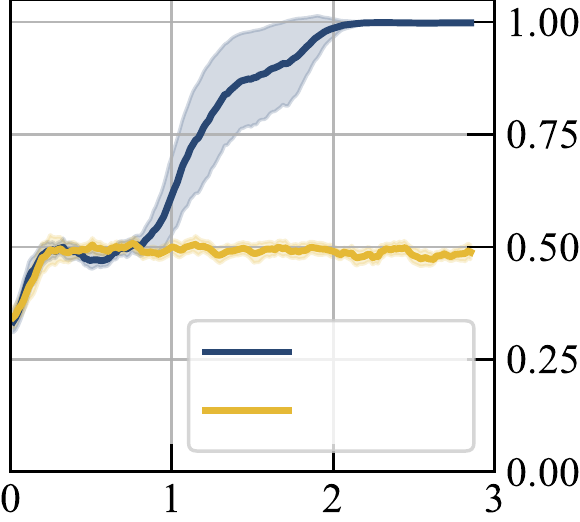}};
  
  \draw (1.7,-0.55) node {Ours}; 
  \draw (1.9,-0.95) node {Myopic};
  
  \draw (3.3,0.1) node[rotate=-90] {Total reward};
  \draw (1,-1.75) node {Frames (millions)};

  \normalsize
  \draw (-2,-2.2) node {(a)};
  \draw (1.1,-2.2) node {(b)};
  \end{tikzpicture}
  \vspace{-3mm}
  \caption{(a) A toy minigrid environment where doors lock upon entry. The task is equally likely to be either \emph{go to blue then red} or \emph{go to blue then green}. (b) A Myopic policy only succeeds in 50\% of tasks while our approach obtains the maximum reward. }
  \label{fig:toy}
  \vspace{-3mm}
\end{figure}

Two recent works have explored how to teach RL agents to follow unseen instructions using temporal logic \cite{kuo2020encoding,leon2020systematic}. Here we discuss the theoretical advantages of our approach over theirs. \citeauthor{kuo2020encoding} propose to learn a policy $\pi^*(a_t|s_1, a_1, ..., s_t,\varphi)$ (using a recurrent neural network) by solving a partially observable problem (i.e., a POMDP). In contrast, we propose to learn a policy $\pi^*(a_t | s_t, \varphi)$ in a Taskable MDP $\mathcal{M}_\Phi$. Since solving MDPs is easier than solving POMDPs (MDPs can be solved in polynomial time whereas POMDPs are undecidable), this gives our approach a theoretical advantage which results in better empirical performance (as shown in Section~\ref{sec:experiments}). 
 

\citet{leon2020systematic} follow a different approach. They instruct agents using a fragment of LTL (which only supports the temporal operator \textit{eventually}) and define a reasoning module that automatically 
returns a proposition to satisfy which makes progress towards solving the task. Thus, the agent only needs to learn a policy $\pi(a|s,p)$ conditioned on the state $s$ and a proposition $p \in \mathcal{P}$. However, this approach is myopic -- it optimizes for solving the next subtask without considering what the agent must do after and, as a result, might converge to suboptimal solutions. This is a common weakness across recent approaches that instruct RL agents \citep[e.g.,][]{DBLP:conf/nips/SohnOL18,DBLP:conf/nips/JiangGMF19,DBLP:conf/iclr/SunWL20}.

As an example, consider (again) the MiniGrid from Figure~\ref{fig:toy}. 
Observe the two red doors at the entrance to each room. These doors automatically lock upon entry so the agent cannot visit both rooms. Suppose the agent has to solve two LTL tasks, uniformly sampled at the beginning of each episode: $\ltleventually (\mathrm{B} \wedge \ltleventually \mathrm{G})$ (\textit{go to a blue square and then to a green square}) or $\ltleventually (\mathrm{B} \wedge \ltleventually \mathrm{R})$ (\textit{go to a blue square and then to a red square}). For both tasks, a myopic approach will tell the agent to first achieve $\ltleventually \mathrm{B}$ (\textit{go to blue}), but doing so without considering where the agent must go after might lead to a dead end (due to the locking doors). In contrast, an approach that learns an optimal policy $\pi^*(a|s,\varphi)$ for $\mathcal{M}_\Phi$ can consistently solve these two tasks (Figure~\ref{fig:toy}(b)).

The theoretical advantages of our approach however comes with a cost. Learning $\pi^*(a|s,\varphi)$ is harder than learning a myopic policy $\pi^*(a|p)$ and, hence, it seems reasonable to expect that a myopic approach will generalize better to unseen instructions. However, we did not observe this behaviour in our experiments.

\section{Model Architecture} \label{sec:arch}


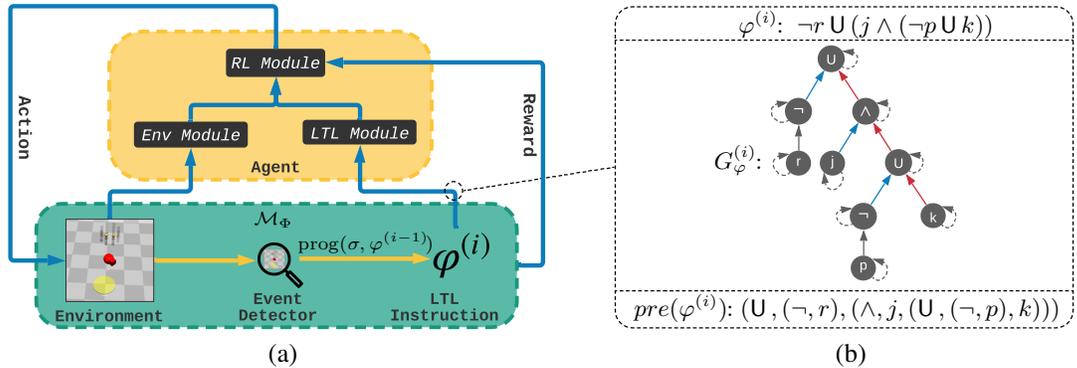
\begin{figure*}[ht]
    \centering
    \begin{tikzpicture}
    \small
    
    \node at (0,0) {\includegraphics[width=1.8\columnwidth]{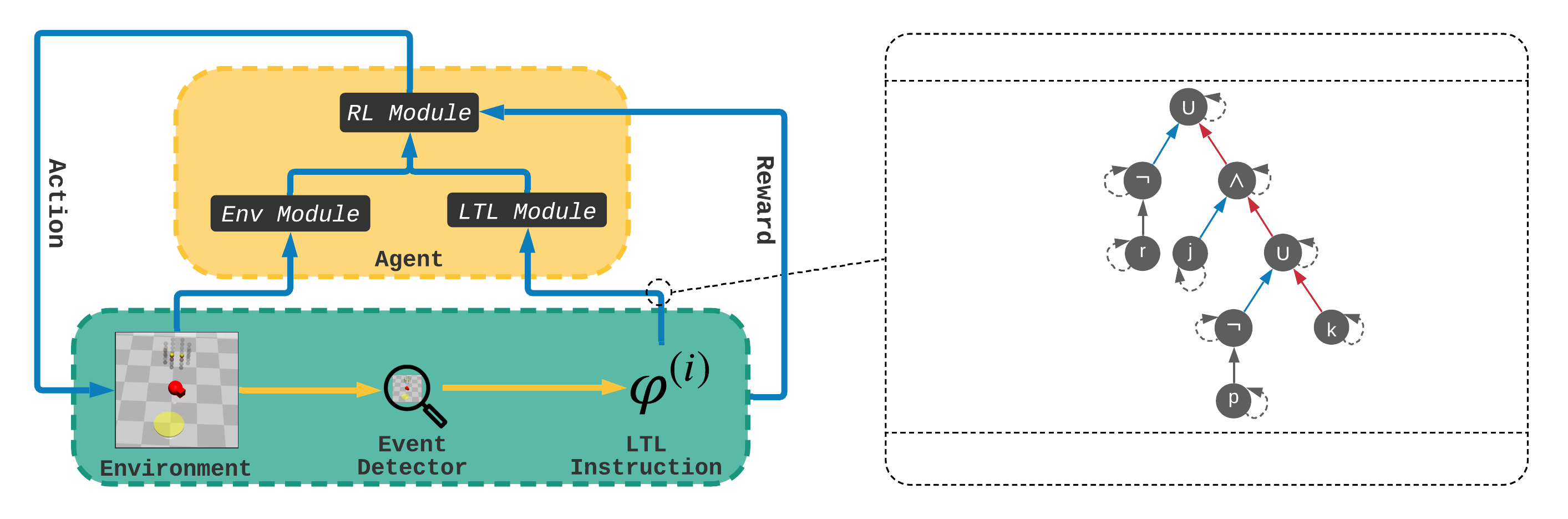}};
    
    \draw (4.3,1.9)  node {$\varphi^{(i)}\!\!:\;\neg r\ltluntil(j\wedge(\neg p \ltluntil k))$};
    \draw (4.05,-1.87)  node {$pre(\varphi^{(i)})\!\!: (\ltluntil, (\neg, r) , (\wedge, j, (\ltluntil, (\neg, p), k)))$};
    \draw (2.6,0.1) node {$G_{\varphi}^{(i)}\!\!: $};
    \scriptsize
    \draw (-3.58,-0.7) node {$\mathcal{M}_{\Phi}$};
    \draw (-2.35,-1) node {$\mprog(\sigma,\varphi^{(i-1)}\!)$};
    
    \normalsize
    \draw (-3.45,-2.5) node {(a)};
    \draw (4.1,-2.5) node {(b)};
    
    \end{tikzpicture}
    \caption{(a) Our RL framework for solving LTL tasks with the Taskable MDP $\mathcal{M}_{\Phi}$ and our agent's modular structure. In each step, the agent receives an environment observation and an LTL instruction as input. LTL instructions are automatically progressed based on signals from the event detectors, allowing a Markov policy to solve non-Markovian tasks. (b) Possible representations of an LTL formula. \emph{Top}: standard LTL syntax. \emph{Middle}: abstract syntax tree. \emph{Bottom}: prefix notation.}
     \label{fig:architecture}
\end{figure*}

In this section, we build on the RL framework with LTL instructions from Section~\ref{sec:multitask-ltl} and explain a way to realize this approach in complex environments using deep RL. 


In each episode, a new LTL task $\varphi$ is sampled. At every step, the environment returns an observation, the event detectors return
a truth assignment $\sigma$ of all propositions in $\mathcal{P}$ and the LTL task is automatically progressed to $\varphi := \mathrm{prog}(\sigma, \varphi)$. We used \texttt{Spot} \cite{spotsoftware} to simplify $\varphi$ to an equivalent form after each progression. The agent then receives both the environment observation and the progressed formula and emits an action via a modular architecture consisting of three trainable components (see Figure~\ref{fig:architecture}(a)):




\textbf{1.\;\;\!\!Env Module:} an environment-dependent model that preprocesses the observations
(e.g., a convolutional or a fully-connected network).

\textbf{2. LTL Module:} a neural encoder for the LTL instructions (discussed below).


\textbf{3.\;\;\;\!\!\!RL Module:}
a module which decides actions to take in the environment, based on observations encoded by the Env Module and the current (progressed) task encoded by the LTL module. While our approach is agnostic to the choice of RL algorithm, in our experiments we opted for \emph{Proximal Policy Optimization} (PPO) \cite{DBLP:journals/corr/SchulmanWDRK17} for its strong generalization performance \cite{DBLP:conf/icml/CobbeKHKS19}. 

\subsection{LTL Module}


\label{sec:ltl_module}
LTL formulas can be encoded through different means, the simplest of which is to apply a sequence model (e.g., LSTM) to the input formula. However, given the tree-structured nature of these formulas, \emph{Graph Neural Networks} (GNNs)  \cite{gnn_gori, scarselli2008graph} may provide a better inductive bias. 
This selection is in line with recent works in programming languages and verification, where GNNs are used to embed the \emph{abstract syntax tree} (AST) of the input program \cite{DBLP:journals/corr/abs-1711-00740, DBLP:conf/nips/SiDRNS18}. 

Out of many incarnations of GNNs, we choose a version of \emph{Relational Graph Convolutional Network} (R-GCN) \cite{DBLP:conf/esws/SchlichtkrullKB18}. R-GCN works on labeled graphs $G=(\mathcal{V}, \mathcal{E}, \mathcal{R})$ with nodes $v, u \in \mathcal{V}$ and typed edges $(v, r, u) \in \mathcal{E}$, where $r\in \mathcal{R}$ is an edge type.
Given $G$ and a set of input node features $\{\boldsymbol{x}^{(0)}_v | \forall v\in \mathcal{V}\}$, the R-GCN maps the nodes to a vector space, through a series of \emph{message passing} steps. At step $t$, the embedding $\boldsymbol{x}^{(t)}_v \in \mathbb{R}^{d^{(t)}}$ of node $v$ is updated by a normalized sum of the transformed feature vectors of neighbouring nodes, followed by an element-wise activation function $\sigma(.)$:
\begin{equation}\label{eq:rgcn}
    \boldsymbol{x}^{(t+1)}_v = \sigma \Bigg( \sum_{r \in \mathcal{R}} \sum_{u \in \mathcal{N}^r_G(v)} \frac{1}{|\mathcal{N}^r_G(v)|}W_r \boldsymbol{x}^{(t)}_u \Bigg),
\end{equation}
where $\mathcal{N}^r_G(v)$ denotes the set of nodes adjacent to $v$ via an edge of type $r\in \mathcal{R}$. Note that different edge types use different weights $W_r$ and weight-sharing is done only for edges of the same type at each iteration.

We represent an LTL formula $\varphi$ as a directed graph ${G_{\varphi}=(V_{\varphi}, E_{\varphi}, R)}$, as shown in Figure \ref{fig:architecture}(b). This is done by first creating $\varphi$'s parse tree, where each subformula is connected to its parent operator via a directed edge, and then adding self-loops for all the nodes. We distinguish between $|\mathcal{R}|=4$ edge types: 1. \emph{Self-loops}, 2. \emph{Unary}: for connecting the subformula of a unary operator to its parent node, 3. \emph{Binary\_left} (\emph{resp.} 4. \emph{Binary\_right}): for connecting the left (\emph{resp.} right) subformula of a binary operator to its parent node. R-GCN performs $T$ message passing steps according to Equation~(\ref{eq:rgcn}) over $G_\varphi$. The inclusion of self-loops is to ensure that the representation of a node at step $t+1$ is also informed by its corresponding representation at step $t$. Due to the direction of the edges in $G_{\varphi}$, the messages flow in a bottom-up manner and after $T$ steps we regard the embedding of the root node of $G_{\varphi}$ as the embedding of $\varphi$.


We experimented with both sequence models and R-GCN. In each case we first create one-hot encodings of the formula tokens (operators and propositions). For sequence models we convert the formula to its prefix notation $pre(\varphi)$ and then replace the tokens with their one-hot encodings. For R-GCN, these encodings serve as input node features $\boldsymbol{x}^{(0)}_v$.

\subsection{Pretraining the LTL Module}\label{sec:pretraining}

Simultaneous training of the LTL Module with the rest of the model can be a strenuous task. We propose to pretrain the LTL module, taking advantage of the environment-agnostic nature of LTL semantics and the agent's modular architecture. Formally, given a target Taskable MDP $\mathcal{M}_{\Phi}$, our aim is to learn useful encodings for formulas in $\Phi$ and later use those encodings in solving $\mathcal{M}_{\Phi}$. We cast the pretraining itself as solving a special kind of Taskable MDP.

\begin{definition}[\simLtl]
\label{def:bootcamp}
Given a set of formulas $\Phi$ and a distribution $\tau$ over $\Phi$, we construct a single-state MDP (without the reward function) $\mathcal{M}_{\varnothing}=\tuple{S,T,A,\prob,\gamma,\mu}$, where $S=\{s_0\}$, $T=\varnothing$, $A =\mathcal{P}$, $\prob(s_0 | s_0, .)=1$, and $\mu(s_0)=1$. The labelling function is given by $L(s_0, p)=\{p\}$. Finally, the $\simLtl$ environment is defined to be the Taskable MDP given by $\mathcal{M}_{\varnothing}$, $\mathcal{P}$, $L$, $\Phi$, and $\tau$.
\end{definition}

Intuitively, the \simLtl task is to progress formulas $\varphi \sim \tau(\Phi)$ to $\true$ in as few steps as possible by setting a single proposition to $\true$ at each step. Hence our scheme is: (1) train to convergence on \simLtl with formula set $\Phi$ and task distribution $\tau$ of the target Taskable MDP $\mathcal{M}_{\Phi}$; (2) transfer the learned LTL Module as the initial LTL Module in $\mathcal{M}_{\Phi}$. While many pretraining schemes are possible, this one involves a simple task which can be viewed as an abstracted version of the downstream task.


Since pretraining does not require interaction with a physical environment, it is more wall-clock efficient than training the full model on the downstream environment. Furthermore, the LTL Module is robust to changes in the environment as long as $\Phi$ and $\tau$ remain the same, thanks to the modular architecture of our model. In Section \ref{sec:experiments} we demonstrate the empirical benefits of pretraining the LTL Module. 

\section{Experiments}
\label{sec:experiments}

We designed our experiments to investigate whether RL agents can learn to solve complex, temporally extended tasks specified in LTL.
Specifically we answer the following questions: \begin{enumerate*}[label=\textbf{(\arabic*)}]
\item \textbf{Performance:} How does our approach fare against baselines that do not utilize LTL progression or are myopic?
\item \textbf{Architecture:} What's the effect of different architectural choices on our model's performance?
\item \textbf{Pretraining:} Does pretraining the LTL Module result in more rapid convergence in novel downstream environments?
\item \textbf{Upward Generalization:} Can the RL agent trained using our approach generalize to larger instructions than those seen in training?
\item \textbf{Continuous Action-Space:} How does our approach perform in a continuous action-space domain?
\end{enumerate*}
\let\thefootnote\relax\footnotetext{Our code and videos of our agents are available at \url{github.com/LTL2Action/LTL2Action}.}
\subsection{Experimental Setup}

We ran experiments across different environments and LTL tasks, where the tasks vary in length and difficulty to form an implicit curriculum. 
To measure how well each approach generalizes to unseen instructions we followed the methodology proposed by \citet{DBLP:conf/icml/CobbeKHKS19}. 
In every episode, the agent faces some LTL task sampled i.i.d.\ from a large set of possible tasks $\Phi$. We evaluate how well the learned policy generalizes to new samples from $\Phi$, the majority of which have not been previously seen, with high probability. We also consider out-of-distribution generalization to larger tasks than those in $\Phi$, which are guaranteed to be unseen.


\subsubsection{Environments}
\label{sec:envs}

We use the following environments in our experiments:

\textbf{\letterWorld:} A $7\times7$ discrete grid environment, similar to \citealt{andreas2017modular}. Out of the 49 squares, 24 are associated with 12 unique propositions/letters (each letter appears twice in the grid, allowing more than one way to satisfy any proposition). At each step the agent can move along the cardinal directions. The agent observes the full grid (and letters) from an egocentric point of view as well as the current LTL task ($\gamma$=0.94, timeout=75 steps).

\textbf{\zoneEnv:} We co-opted OpenAI's Safety Gym \cite{safetygym} which has a continuous action-space. Our environment (Figure \ref{fig:safety_results}(a)) is a walled 2D plane with 8 circles (2 of each colour), called ``zones,'' that correspond to task propositions. We use Safety Gym's Point robot with actions for steering and forward/backward acceleration.
It observes lidar information towards the zones and other sensory data (e.g., accelerometer, velocimeter). The zones and the robot are randomly positioned on the plane at the start of each episode and the robot has to visit and/or avoid certain zones based on the LTL task 
($\gamma$=0.998, timeout=1000 steps).

\textbf{\simLtl:} The Taskable MDP from Section~\ref{sec:arch}, only used to pretrain the LTL Module ($\gamma$=0.9, timeout=75 steps).


\subsubsection{Tasks}
\label{sec:tasks}


Our experiments consider two LTL task spaces, where tasks are randomly sampled via procedural generation. We provide a high-level description of the two task spaces, with more details in Appendix~\ref{app:ltl_tasks}.


\textbf{Partially-Ordered Tasks:} A task consists of multiple sequences of propositions which can be solved in parallel. However, the propositions within each sequence must be satisfied in order. For example, a possible task (specified informally in English) is: \emph{``satisfy $C$, $A$, $B$ in that order, and satisfy $D$, $A$ in that order"} -- where one valid solution would be to satisfy $D$, $C$, $A$, $B$ in that order. The number of possible unique tasks is over $5 \times 10^{39}$.

\textbf{Avoidance Tasks:} This set of tasks is similar to Partially-Ordered Tasks, but includes propositions that must also be avoided (or else the task is failed). The propositions to avoid change as different parts of the task are solved. The number of possible unique tasks is over $970$ million.

Note that the formulas we consider contain up to 75 tokens (propositions and operators) in training and 210 tokens in the upward generalization experiments, while past related works only considered up to 20 tokens \citep{kuo2020encoding, leon2020systematic}.

\subsubsection{Our Methods and Baselines}
We experimented with three variants of our approach, all exploiting LTL progression and utilizing PPO for policy optimization. They differed in the type of LTL Module used, namely: GNN, GRU, and LSTM. In our plots, we refer to these approaches as \gnnp, \grup, and, \lstmp, respectively. 
Details about neural network architectures and PPO hyperparameters can be found in Appendix Sections \ref{app:networks}, \ref{app:hyperparams}, respectively.

We compared our method against three baselines. The \emph{No LTL} baseline ignores the LTL instructions, but learns a non-stationary policy using an LSTM. This baseline tells us if the agent can learn a policy that works well regardless of the LTL instruction. The \emph{GRU} baseline is inspired by \citet{kuo2020encoding}. This approach learns a policy that considers the LTL instructions but does not progress the formula over time. Instead, it learns a non-stationary policy encoded using a GRU (as discussed in Section~\ref{sec:ltl_discussion}).


Lastly, the \emph{Myopic} baseline was inspired by \citet{leon2020systematic} and other similar approaches \citep[e.g.,][]{andreas2017modular, DBLP:conf/icml/OhSLK17,DBLP:conf/icra/XuNZGGFS18, DBLP:conf/nips/SohnOL18,DBLP:conf/iclr/SunWL20}. In this baseline, a reasoning technique is used to tell the agent which propositions to achieve next in order to solve the LTL task. Specifically, the agent observes whether making a particular proposition $\true$ would (a) progress the current formula, (b) have no effect, or (c) make it unsatisfiable. Given these observations, the agent then learns a Markovian policy.
Note that this approach might converge to suboptimal solutions (see Section~\ref{sec:ltl_discussion}).




\subsection{Results}
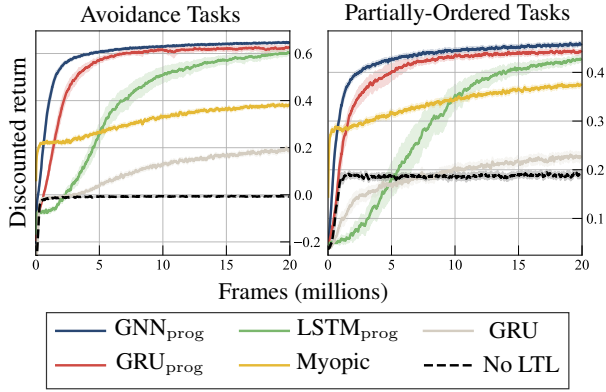
\begin{figure}[t]
    \centering
    \begin{tikzpicture}
    \small
    
    \node at (2.5,0) {\includegraphics[width=0.95\columnwidth]{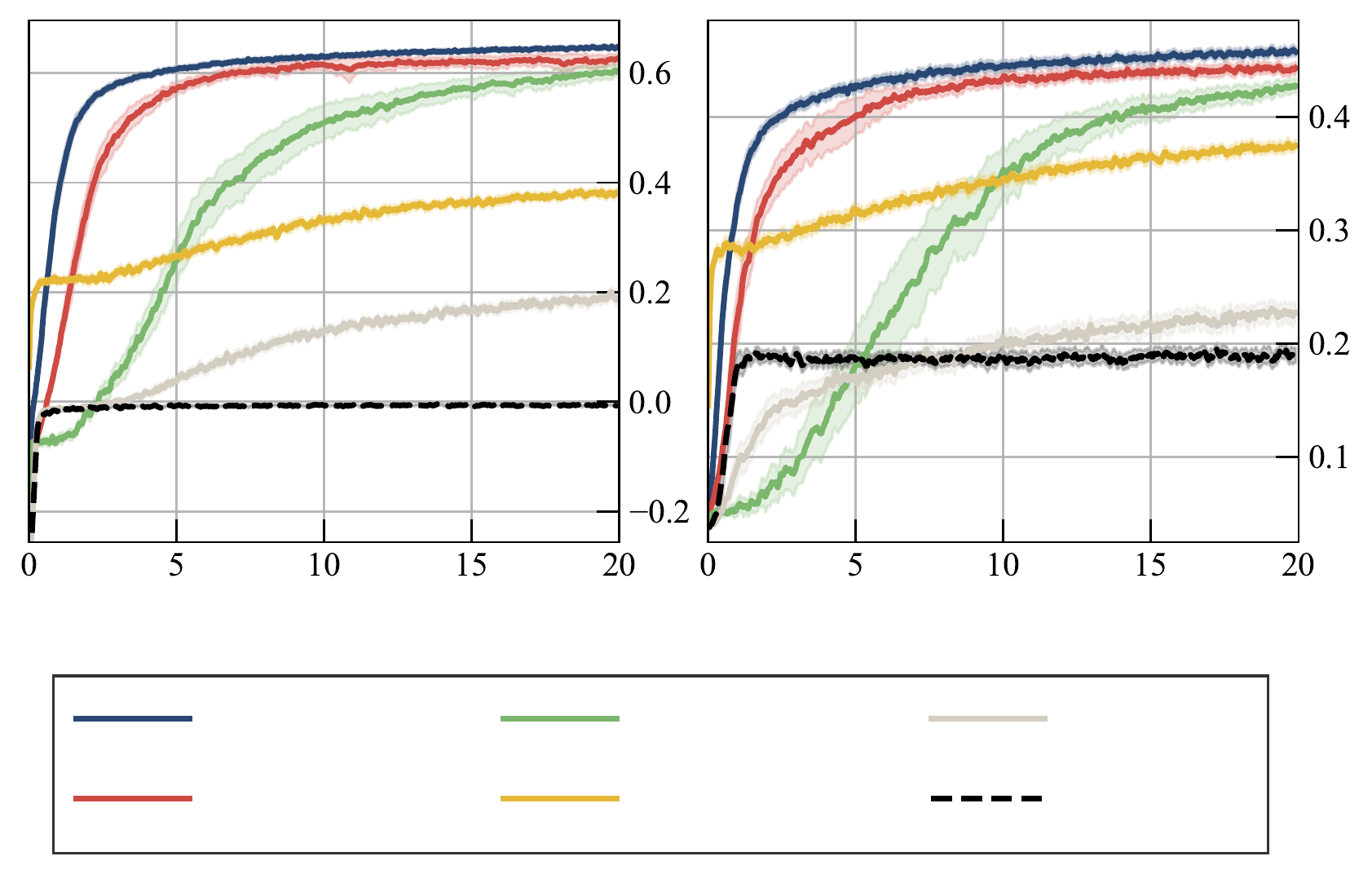}};
    \draw (0.4,2.6) node {Avoidance Tasks}; 
    \draw (4.35,2.6) node {Partially-Ordered Tasks}; 
    \draw (-1.5,0.9) node[rotate=90] {Discounted return};
    \draw (2.3,-1.1) node {Frames (millions)}; 
    
    \draw (0.4,-1.6) node {\gnnp};
    \draw (0.4,-2.05) node {\grup}; 
    
    \draw (2.9,-1.6) node {\lstmp};
    \draw (2.7,-2.05) node {Myopic}; 
    
    \draw (5.1,-1.6) node {GRU};
    \draw (5.2,-2.05) node {No LTL}; 
    
    \end{tikzpicture}
  \vspace{-8mm}
  \caption{ Our approaches using LTL progression (marked by $\bullet_\mathrm{prog}$) outperformed other baselines on \letterWorld. We report \emph{discounted return} over the duration of training (averaged over 30 seeds, with 90\% confidence intervals).
  } 
  \label{fig:letter_world_graph}
\end{figure}

\begin{figure}[t]
    \centering
    \begin{tikzpicture}
    \small
    
    \node at (2.5,0) {\includegraphics[width=0.95\columnwidth]{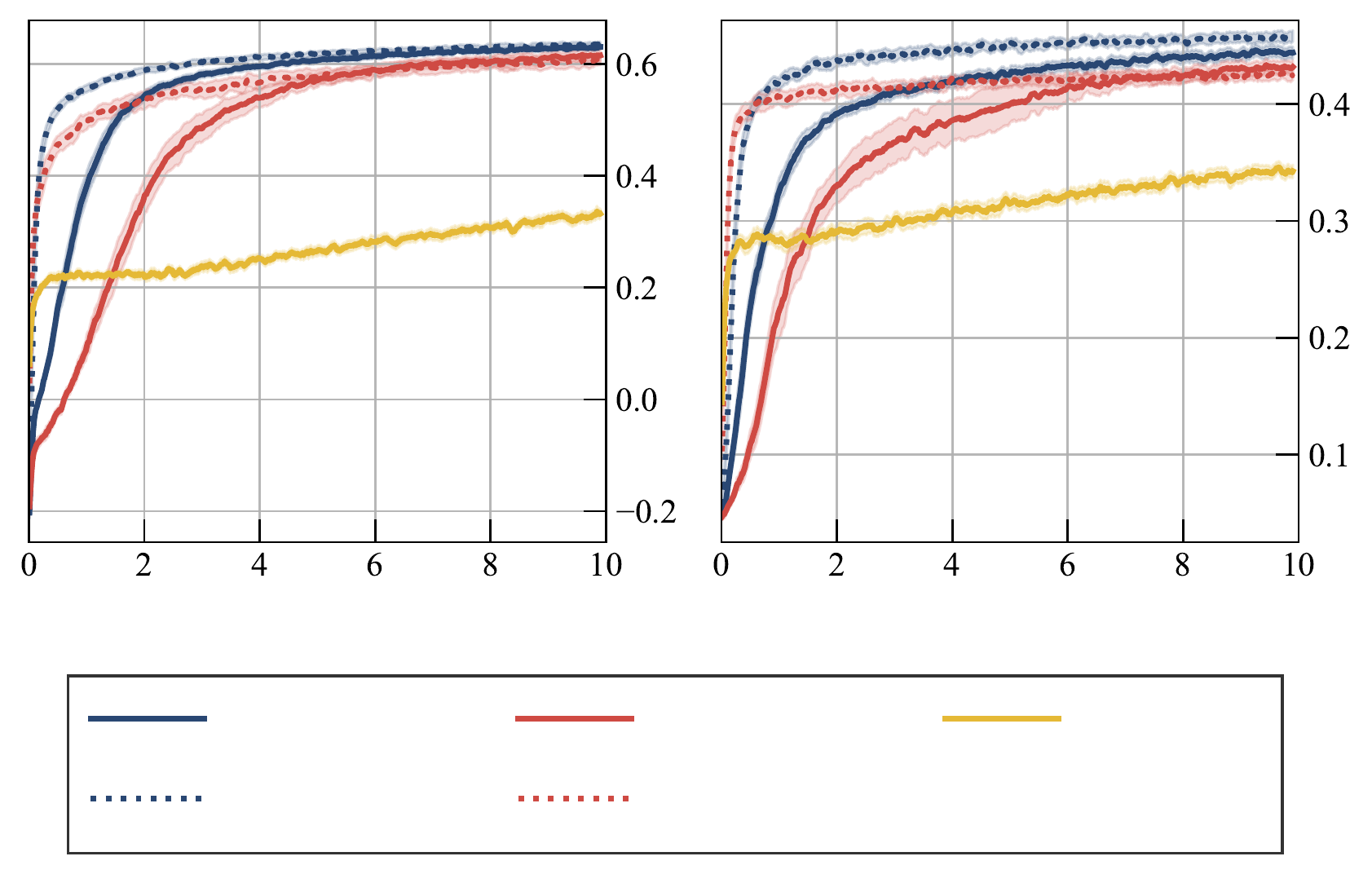}};
    \draw (0.4,2.6) node {Avoidance Tasks}; 
    \draw (4.35,2.6) node {Partially-Ordered Tasks}; 
    \draw (-1.5,0.9) node[rotate=90] {Discounted return};
    \draw (2.3,-1.1) node {Frames (millions)}; 
    
    \draw (0.4,-1.6) node {\gnnp};
    \draw (0.4,-2.05) node {\gnnpp}; 
    
    \draw (2.9,-1.6) node {\grup};
    \draw (2.9,-2.05) node {\grupp}; 
    
    \draw (5.25,-1.6) node {Myopic};
    
    \end{tikzpicture}
  \vspace{-8mm}
  \caption{
  Pretrained LTL models (marked by $\bullet^\mathrm{pre}$) showed better sample-efficiency than non-pretrained versions on \letterWorld. The \emph{Myopic} baseline is shown for comparison. We report \emph{discounted return} over the duration of training (averaged over 30 seeds, with 90\% confidence intervals).}
  \label{fig:transfer_graph}
  \vspace{-2mm}
\end{figure}
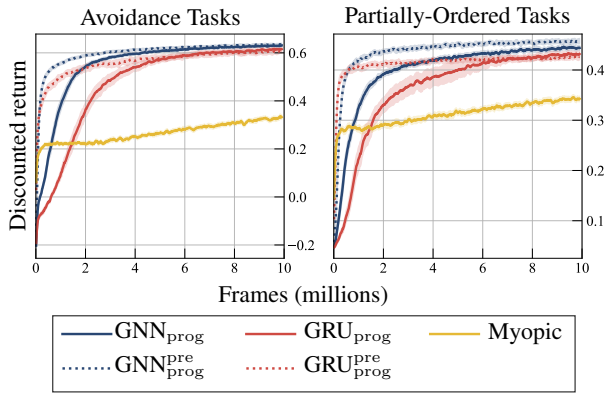
We conduct our experiments on \letterWorld, except for the continuous action-space tests where we used \zoneEnv. 


\begin{figure}[h]
    \centering
    \begin{tikzpicture}
    \small
    
    \node at (-3,0.07) {\includegraphics[width=0.42\columnwidth,valign=c]{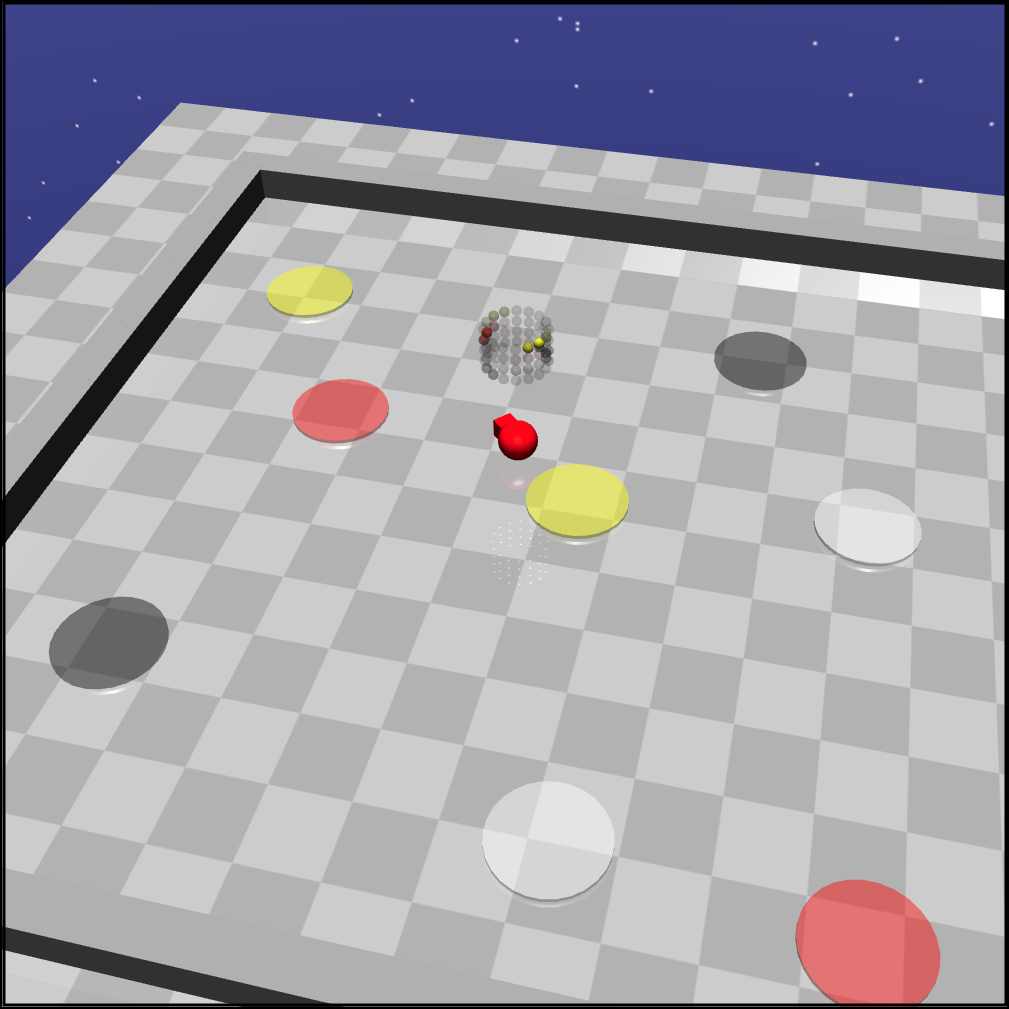}};
    \node at (0.85,0) {\includegraphics[width=0.5\columnwidth,valign=c]{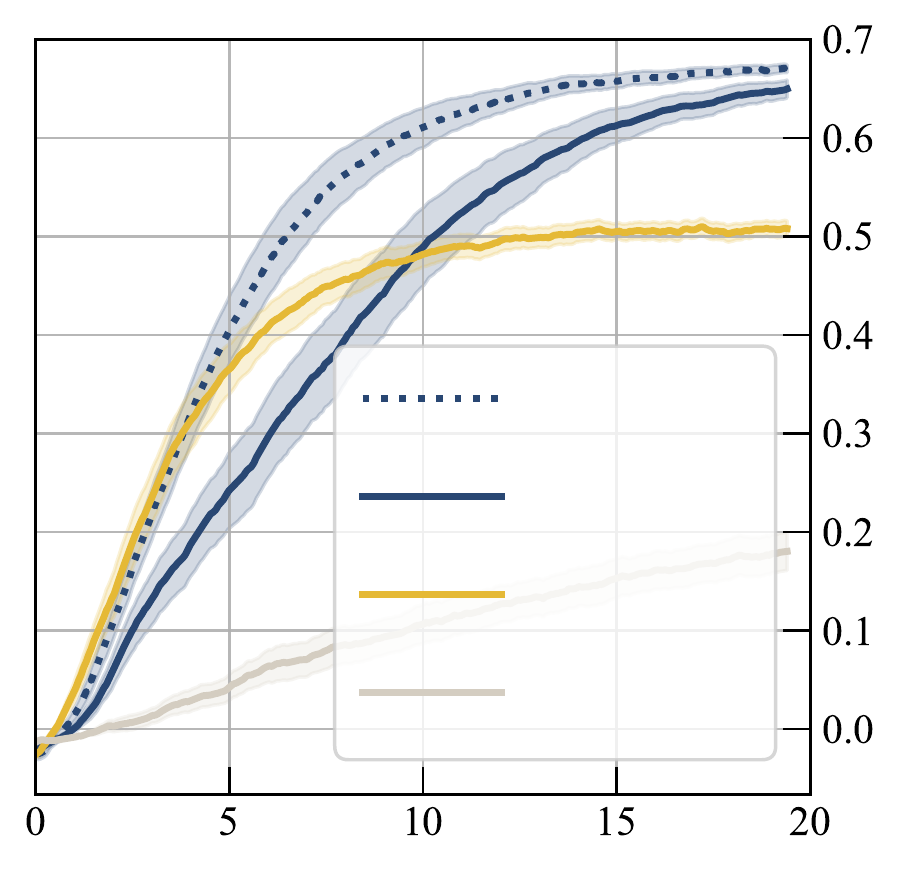}};
    
    \draw (1.75,0.15) node {\gnnpp}; 
    \draw (1.75,-0.3) node {\gnnp};
    \draw (1.65,-0.7) node {Myopic};
    \draw (1.55,-1.2) node {GRU};
    
    \draw (3,0) node[rotate=-90] {Discounted return};
    \draw (1,-2.1) node {Frames (millions)}; 
    
    \normalsize
    \draw (-3,-2.55)  node {(a)};
    \draw (0.8,-2.55) node {(b)};
    
    \end{tikzpicture}
  \vspace{-3mm}
  \caption{(a) The \texttt{ZoneEnv} continuous control environment with coloured zones as LTL propositions. Tasks involve reaching zones of certain colours in the correct order (while avoiding zones of the incorrect colour in the Avoidance Task). (b)
  Both GNN architectures outperformed the \emph{Myopic} and GRU without progression baselines on \texttt{ZoneEnv} on Avoidance Tasks. Pretraining resulted in faster convergence. We report \emph{discounted return} over the duration of training (averaged over 30 seeds, with 90\% confidence intervals).}
  \label{fig:safety_results}
  \vspace{-3mm}
\end{figure}
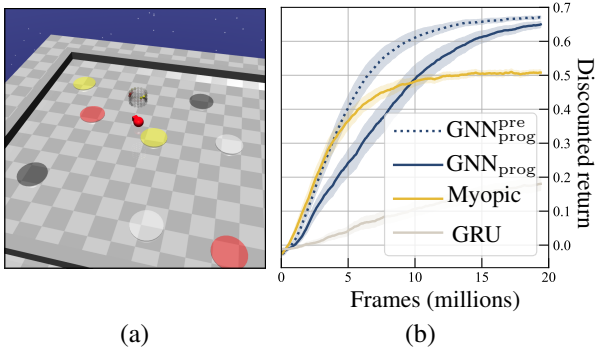


\begin{table}[t]
\small
\centering
\caption{
Trained RL agents are evaluated on the training distribution of tasks, as well as out-of-distribution tasks with increased depth of sequences, and increased number of conjuncts. In each entry, we report \textit{total reward} and \textit{discounted return} (in parentheses) averaged over 30 seeds and 100 episodes per seed. The highest value over all approaches is bolded.}
\vspace{1mm}
\label{table:generalization}

\begin{tabular}{*{4}c}
\cmidrule[\heavyrulewidth]{2-4} 
& \multicolumn{1}{c}{I.I.D} &
  \multicolumn{1}{c}{$\uparrow$ Depth} &
  \multicolumn{1}{c}{$\uparrow$ Conjuncts} \\ \cmidrule[\heavyrulewidth]{2-4}
& \multicolumn{3}{c}{(a) Avoidance Tasks} \\ \cmidrule(lr){2-4} 
\multicolumn{1}{l}{\gnnpp} &  0.97(0.65) & \textbf{0.99}(\textbf{0.35}) & \textbf{0.64}(0.18)  \\ 
\multicolumn{1}{l}{\gnnp} & \textbf{0.98}(\textbf{0.66}) & 0.98(0.33)   & 0.57(0.14) \\ 

\multicolumn{1}{l}{\grup} & 0.89(0.63) & 0.42(0.07) & 0.58(\textbf{0.25})  \\ 
\multicolumn{1}{l}{GRU}      & 0.32(0.22) & -0.03(-0.01) & -0.35(-0.16) \\ 
\multicolumn{1}{l}{Myopic}  & 0.88(0.50) & 0.71(0.07) & 0.58(0.11) \\

 \cmidrule{2-4} &
  \multicolumn{3}{c}{(b) Partially-Ordered Tasks} \\ 
  \cmidrule{2-4}
\multicolumn{1}{l}{\gnnpp} & \textbf{1.0}(\textbf{0.48}) &  \textbf{0.98}(\textbf{0.0088}) & \textbf{0.99}(\textbf{0.0380})  \\ 
\multicolumn{1}{l}{\gnnp}  & \textbf{1.0}(0.47) & 0.97(0.0074)  & 0.98(0.0340) \\ 
\multicolumn{1}{l}{\grup}  & \textbf{1.0}(0.46)  & 0.29(0.0005)  & \textbf{0.99}(0.0252) \\ 
\multicolumn{1}{l}{GRU}   & 0.87(0.24) & 0.03(0.0000) & 0.54(0.0075) \\ 
\multicolumn{1}{l}{Myopic} & \textbf{1.0}(0.40) & 0.94(0.0042) & \textbf{0.99}(0.0221) \\\bottomrule
\end{tabular}
\vspace{-2mm}
\end{table}

\textbf{Performance}
Figure~\ref{fig:letter_world_graph} shows the results on Partially-Ordered and Avoidance Tasks when tested on i.i.d.\ samples from $\Phi$. The results show that: (1) LTL progression significantly improved generalization, (2) A compositional architecture such as GNN learned to encode LTL formulas better than sequential ones such as GRU and LSTM, and (3) The myopic baseline initially learned quickly, but it was eventually outperformed by all our methods.


\textbf{Pretraining}
We investigated the effects of pretraining the LTL Module (Section~\ref{sec:pretraining}). Figure~\ref{fig:transfer_graph} shows the results for GNN and GRU encoders, with and without pretraining, as well as the myopic baseline. We observed that pretraining the LTL Module accelerated learning for both encoders. 



\textbf{Upward Generalization} 
%
The preceding experiments indicate that an RL agent trained with our approach generalizes to new, incoming tasks from a large, but fixed training distribution $\Phi$. Here, we consider \emph{upward generalization} to larger tasks than seen in training. 

We evaluated trained agents on Partially-Ordered and Avoidance Tasks with: (a) longer sequences, and (b) more conjuncts (i.e., more tasks to be completed in parallel) than seen in training.
For Avoidance Tasks, we increased the max depth of formulas from 3 (in training) to 6, and the max number of conjuncts from 2 to 3. For Partially-Ordered Tasks, we increased the max depth from 5 to 15, and the max number of conjuncts from 4 to 12.

We report the generalization performance of various baselines in Table~\ref{table:generalization}. The myopic and LTL progression-based approaches significantly outperformed the GRU baseline without progression, suggesting that decomposing the task is essential for generalization. Pretraining also marginally improved the GNN (with progression) baseline. We highlight the impact of architecture on generalization -- GNN outperformed GRU in most cases. This aligns with other works showing scalability of GNNs to larger formulas \cite{selsam2018learning, vaezipoor2020learning}. Note that upward generalization on conjuncts for Avoidance tasks is particularly challenging since only up to 2 conjuncts were observed in training.

\textbf{Continuous Action-Space}
Figure~\ref{fig:safety_results} shows the results on \zoneEnv for a reduced version of the Avoidance Task. We note that our approaches ({\gnnpp} and {\gnnp}) solved almost all tasks, however \emph{Myopic} and the GRU baseline without progression failed to solve many tasks within the timeout. 
These results reaffirm the generalizability of our approach on a continuous environment.

\section{Discussion}

Our experiments demonstrated the following key findings: \textbf{(a)} encoding full task instructions converges to better solutions than myopic methods; \textbf{(b)} LTL progression improves learning and generalization; \textbf{(c)} LTL semantics can be pretrained to improve downstream learning; \textbf{(d)} our method can zero-shot generalize to new instructions significantly larger than those seen in training.

We note that architecture is an important factor for encoding LTL. GNNs appear to more effectively encode the compositional syntax of LTL, whereas, GRUs are more wall-clock efficient (by roughly 2 to 3$\times$) due to the overhead of constructing abstract syntax trees for GNNs.

Our results are encouraging and open several directions for future work. This includes exploring ways to build general LTL models which fully capture LTL semantics without assuming access to a distribution of tasks. Similar to most works focusing on formal language in RL (e.g. \citealt{DBLP:conf/atal/IcarteKVM18, alur2019, leon2020systematic}), we assume a noise-free \emph{labelling function} is available to identify high-level domain features. An important question is whether this labelling function can be learned, and how an RL agent can handle the resultant uncertainty.


In this work, we investigated generalization to new formulas over a fixed set of propositions. However, it is also interesting to study how to generalize to formulas with \emph{new propositions}. We note that some existing works have tackled this setting \cite{hill2021grounded, leon2020systematic, lake2019compositional}. One way of extending our framework to also generalize to \emph{unseen} propositions is to encode the propositions using some feature representation other than a one-hot encoding. We include some preliminary experiments in Appendix~\ref{sec:object-gen} showing that, by changing the feature representation of the propositional symbols, our framework is indeed able to generalize to tasks with unseen objects and propositions. But further investigation is needed. 

\section{Related Work}


This paper builds on past work in RL which explores using LTL (or similar formal languages) for reward function specification, decomposition, or shaping \citep[e.g.,][]{aksaray2016q,li2016temporal, littman2017environment,DBLP:conf/icml/IcarteKVM18,li2018policy,camacho2017decision,DBLP:conf/ijcai/CamachoIKVM19,deeplcrl,alur2019,DBLP:conf/ijcai/0005T19,hasanbeig2018logically,hasanbeig2020deep,DBLP:conf/kr/GiacomoFIPR20,DBLP:conf/aaai/GiacomoIFP20,jiang2020temporal}. However, most of these methods are limited to learning a single, fixed task in LTL. \citet{DBLP:conf/atal/IcarteKVM18} explicitly focuses on learning multiple LTL tasks optimally, but their approach is unable to generalize to unseen tasks and may have to learn an exponential number of policies in the length of the largest formula.



In this work, we consider a multitask setting in which a new task is sampled each episode from a large task space. Our motivation is to enable an agent to solve unseen tasks without further training, similar in spirit to previous works (e.g. \citealt{andreas2017modular, DBLP:conf/icra/XuNZGGFS18, DBLP:conf/icml/OhSLK17, DBLP:conf/nips/SohnOL18}), some of which also considers temporal logic tasks (\citealt{leon2020systematic, kuo2020encoding}). A common theme in all the previous listed works (except for one: \citealt{kuo2020encoding}) is to decompose large tasks into independent, sequential subtasks, however this often performs suboptimally in solving the full task. We instead consider the full LTL task, as also adopted by \citealt{kuo2020encoding}. We additionally propose to use LTL progression to enable standard, Markovian learning -- drastically improving both sample and wall-clock efficiency -- and pretraining the LTL module to accelerate learning. Note that \citealt{kuo2020encoding} do not encode LTL formulas, but instead compose neural networks to mirror the formula structure, which is incompatible with LTL pretraining as it is environment-dependent.



Other representations of task specifications for RL have also been proposed, including programs \cite{DBLP:conf/iclr/SunWL20, fasel2009task, DBLP:journals/corr/DenilCCSF17}, policy sketches \cite{andreas2017modular}, and natural language (\citealt{DBLP:conf/nips/JiangGMF19, DBLP:conf/iclr/BahdanauHLHHKG19}; see \citealt{DBLP:conf/ijcai/LuketinaNFFAGWR19} for an extensive survey). Finally, note that it is possible to automatically translate natural language instructions into LTL \citep[e.g.,][]{scheutz2009,brunello2019synthesis,wang2020learning}. 









\section{Conclusion}

Creating learning agents that understand and follow open-ended human instructions is a challenging problem with significant real-world applicability. Part of the difficulty stems from the need for large training sets of instructions and associated rewards. 
In this work, we trained an RL agent to follow temporally extended instructions specified in the formal language, LTL. 
The compositional syntax of LTL allowed us to generate massive training data, automatically.
We theoretically motivate a novel approach to multitask RL using neural encodings of LTL instructions and LTL progression. Our experiments on discrete and continuous domains demonstrated robust generalization across procedurally generated task sets, outperforming a prevalent myopic approach.

LTL is a popular specification language for 
synthesis and verification, and is 
used for robot tasking and human-robot interaction  
\citep[e.g.,][]{scheutz2009,DBLP:conf/iros/RamanFK12,li2016temporal,DBLP:conf/mrs/MoarrefK17,DBLP:journals/arobots/MoarrefK20,DBLP:conf/cdc/KasenbergS17,DBLP:conf/nips/ShahKSL18,DBLP:journals/ral/ShahLS20}. 
We believe the contributions in this paper have the potential for broad impact within these communities, and could be adapted for other structured languages.

\removehide{
\added{These results are encouraging and open several directions for future work. Some options include exploring ways to build and incorporate distribution-agnostic LTL models which fully capture LTL semantics.
Also, our approach assumes that the agent has access to a set of accurate event detectors (also known as a labelling function). This opens the question of whether we can learn such detectors and how the agent should behave if those detectors have some degree of noise. }
\remove{Finally, since the long-term aspiration of this line of research is to instruct RL agent, it makes sense to consider model-based RL approaches \citep[e.g.,][]{DBLP:conf/icml/CorneilGB18,DBLP:journals/corr/abs-2012-02419} and exploit existing approaches to transform natural language into LTL \citep[e.g.,][]{scheutz2009,wang2020learning}.}
}

\subsection*{Acknowledgements}

We gratefully acknowledge funding from the Natural Sciences and Engineering Research Council of Canada (NSERC), the Canada CIFAR AI Chairs Program, and Microsoft Research. The third author also gratefully acknowledges funding from ANID (Becas Chile). Resources used in preparing this research were provided, in part, by the Province of Ontario, the Government of Canada through CIFAR, and companies sponsoring the Vector Institute for Artificial Intelligence \url{www.vectorinstitute.ai/partners}. We
thank the Schwartz Reisman Institute for Technology and Society for
providing a rich multi-disciplinary research environment. Additionally, we thank our anonymous reviewers for their feedback which led to several improvements in this paper. 

\bibliography{main.bib}

\providecommand{\Proceedings}{Proceedings\xspace}
  \providecommand{\International}{International\xspace}
  \providecommand{\Conference}{Conference\xspace}
  \providecommand{\Artificial}{Artificial\xspace}
  \providecommand{\Intelligence}{Intelligence\xspace}
  \providecommand{\AI}{Artificial Intelligence}
  \providecommand{\Scheduling}{Sched.\xspace} \providecommand{\ofthe}{of
  the\xspace} \providecommand{\longshortnopar}[2]{#1}
  \providecommand{\longshort[2]}{#1 (#2)}
\begin{thebibliography}{61}
\providecommand{\natexlab}[1]{#1}
\providecommand{\url}[1]{\texttt{#1}}
\expandafter\ifx\csname urlstyle\endcsname\relax
  \providecommand{\doi}[1]{doi: #1}\else
  \providecommand{\doi}{doi: \begingroup \urlstyle{rm}\Url}\fi

\bibitem[Aksaray et~al.(2016)Aksaray, Jones, Kong, Schwager, and
  Belta]{aksaray2016q}
Aksaray, D., Jones, A., Kong, Z., Schwager, M., and Belta, C.
\newblock {Q-learning for Robust Satisfaction of Signal Temporal Logic
  Specifications}.
\newblock In \emph{\longshort{\Proceedings \ofthe 55th IEEE Annual \Conference
  on Decision and Control}{CDC}}, pp.\  6565--6570. IEEE, 2016.

\bibitem[Allamanis et~al.(2018)Allamanis, Brockschmidt, and
  Khademi]{DBLP:journals/corr/abs-1711-00740}
Allamanis, M., Brockschmidt, M., and Khademi, M.
\newblock {Learning to Represent Programs with Graphs}.
\newblock In \emph{\longshort{\Proceedings \ofthe 6th \International
  \Conference on Learning Representations}{ICLR}}, 2018.

\bibitem[Andreas et~al.(2017)Andreas, Klein, and Levine]{andreas2017modular}
Andreas, J., Klein, D., and Levine, S.
\newblock {Modular Multitask Reinforcement Learning with Policy Sketches}.
\newblock In \emph{\longshort{\Proceedings \ofthe 34th \International
  \Conference on Machine Learning}{ICML}}, pp.\  166--175. PMLR, 2017.

\bibitem[Bacchus \& Kabanza(2000)Bacchus and
  Kabanza]{DBLP:journals/ai/BacchusK00}
Bacchus, F. and Kabanza, F.
\newblock {Using Temporal Logics to Express Search Control Knowledge for
  Planning}.
\newblock \emph{Artificial Intelligence}, 116\penalty0 (1-2):\penalty0
  123--191, 2000.

\bibitem[Bahdanau et~al.(2018)Bahdanau, Hill, Leike, Hughes, Hosseini, Kohli,
  and Grefenstette]{DBLP:conf/iclr/BahdanauHLHHKG19}
Bahdanau, D., Hill, F., Leike, J., Hughes, E., Hosseini, A., Kohli, P., and
  Grefenstette, E.
\newblock {Learning to Understand Goal Specifications by Modelling Reward}.
\newblock In \emph{\longshort{\Proceedings \ofthe 6th \International
  \Conference on Learning Representations}{ICLR}}, 2018.

\bibitem[Baier \& Katoen(2008)Baier and Katoen]{DBLP:books/daglib/0020348}
Baier, C. and Katoen, J.
\newblock \emph{{Principles of Model Checking}}.
\newblock {MIT} Press, 2008.

\bibitem[Brunello et~al.(2019)Brunello, Montanari, and
  Reynolds]{brunello2019synthesis}
Brunello, A., Montanari, A., and Reynolds, M.
\newblock {Synthesis of LTL formulas from natural language texts: State of the
  art and research directions}.
\newblock In \emph{\longshort{\Proceedings \ofthe 26th \International Symposium
  on Temporal Representation and Reasoning}{TIME}}. Schloss
  Dagstuhl-Leibniz-Zentrum fuer Informatik, 2019.

\bibitem[Camacho et~al.(2017)Camacho, Chen, Sanner, and
  McIlraith]{camacho2017decision}
Camacho, A., Chen, O., Sanner, S., and McIlraith, S.~A.
\newblock {Decision-Making with non-Markovian Rewards: From LTL to
  Automata-based Reward Shaping}.
\newblock In \emph{\longshort{\Proceedings \ofthe 3rd Multi-disciplinary
  \Conference on Reinforcement Learning and Decision}{RLDM}}, pp.\  279--283,
  2017.

\bibitem[Camacho et~al.(2019)Camacho, {Toro Icarte}, Klassen, Valenzano, and
  McIlraith]{DBLP:conf/ijcai/CamachoIKVM19}
Camacho, A., {Toro Icarte}, R., Klassen, T.~Q., Valenzano, R., and McIlraith,
  S.~A.
\newblock {LTL and Beyond: Formal Languages for Reward Function Specification
  in Reinforcement Learning}.
\newblock In \emph{\longshort{\Proceedings \ofthe 28th \International Joint
  \Conference on \AI{}}{IJCAI}}, volume~19, pp.\  6065--6073, 2019.

\bibitem[Chaplot et~al.(2018)Chaplot, Sathyendra, Pasumarthi, Rajagopal, and
  Salakhutdinov]{DBLP:conf/aaai/ChaplotSPRS18}
Chaplot, D.~S., Sathyendra, K.~M., Pasumarthi, R.~K., Rajagopal, D., and
  Salakhutdinov, R.
\newblock {Gated-Attention Architectures for Task-Oriented Language Grounding}.
\newblock In \emph{\longshort{\Proceedings \ofthe 32nd AAAI \Conference on
  \AI{}}{AAAI}}, volume~32, 2018.

\bibitem[Chevalier-Boisvert et~al.(2018)Chevalier-Boisvert, Willems, and
  Pal]{gym_minigrid}
Chevalier-Boisvert, M., Willems, L., and Pal, S.
\newblock {Minimalistic Gridworld Environment for OpenAI Gym}.
\newblock \url{https://github.com/maximecb/gym-minigrid}, 2018.

\bibitem[Co-Reyes et~al.(2019)Co-Reyes, Gupta, Sanjeev, Altieri, DeNero,
  Abbeel, and Levine]{DBLP:conf/iclr/Co-ReyesGSAADAL19}
Co-Reyes, J.~D., Gupta, A., Sanjeev, S., Altieri, N., DeNero, J., Abbeel, P.,
  and Levine, S.
\newblock {Meta-Learning Language-Guided Policy Learning}.
\newblock In \emph{\longshort{\Proceedings \ofthe 7th \International
  \Conference on Learning Representations}{ICLR}}, 2019.

\bibitem[Cobbe et~al.(2019)Cobbe, Klimov, Hesse, Kim, and
  Schulman]{DBLP:conf/icml/CobbeKHKS19}
Cobbe, K., Klimov, O., Hesse, C., Kim, T., and Schulman, J.
\newblock {Quantifying Generalization in Reinforcement Learning}.
\newblock In \emph{\longshort{\Proceedings \ofthe 36th \International
  \Conference on Machine Learning}{ICML}}, pp.\  1282--1289. PMLR, 2019.

\bibitem[de~Giacomo et~al.(2020{\natexlab{a}})de~Giacomo, Favorito, Iocchi,
  Patrizi, and Ronca]{DBLP:conf/kr/GiacomoFIPR20}
de~Giacomo, G., Favorito, M., Iocchi, L., Patrizi, F., and Ronca, A.
\newblock {Temporal Logic Monitoring Rewards via Transducers}.
\newblock In \emph{\longshort{\Proceedings \ofthe 17th \International
  \Conference on Principles of Knowledge Representation and Reasoning}{KR}},
  volume~17, pp.\  860--870, 2020{\natexlab{a}}.

\bibitem[de~Giacomo et~al.(2020{\natexlab{b}})de~Giacomo, Iocchi, Favorito, and
  Patrizi]{DBLP:conf/aaai/GiacomoIFP20}
de~Giacomo, G., Iocchi, L., Favorito, M., and Patrizi, F.
\newblock {Restraining Bolts for Reinforcement Learning Agents}.
\newblock In \emph{\longshort{\Proceedings \ofthe 34th AAAI \Conference on
  \AI{}}{AAAI}}, volume~34, pp.\  13659--13662, 2020{\natexlab{b}}.

\bibitem[Denil et~al.(2017)Denil, Colmenarejo, Cabi, Saxton, and
  de~Freitas]{DBLP:journals/corr/DenilCCSF17}
Denil, M., Colmenarejo, S.~G., Cabi, S., Saxton, D., and de~Freitas, N.
\newblock {Programmable Agents}.
\newblock \emph{arXiv preprint arXiv:1706.06383}, 2017.

\bibitem[Duret-Lutz et~al.(2016)Duret-Lutz, Lewkowicz, Fauchille, Michaud,
  Renault, and Xu]{spotsoftware}
Duret-Lutz, A., Lewkowicz, A., Fauchille, A., Michaud, T., Renault, E., and Xu,
  L.
\newblock {Spot 2.0—A Framework for LTL and $\omega$-Automata Manipulation}.
\newblock In \emph{\longshort{\Proceedings \ofthe 14th \International Symposium
  on Automated Technology for Verification and Analysis}{ATVA}}, pp.\
  122--129. Springer, 2016.

\bibitem[Dzifcak et~al.(2009)Dzifcak, Scheutz, Baral, and
  Schermerhorn]{scheutz2009}
Dzifcak, J., Scheutz, M., Baral, C., and Schermerhorn, P.
\newblock {What to Do and How to Do It: Translating Natural Language Directives
  Into Temporal and Dynamic Logic Representation for Goal Management and Action
  Execution}.
\newblock In \emph{\longshort{\Proceedings \ofthe 2009 IEEE \International
  \Conference on Robotics and Automation}{ICRA}}, pp.\  4163--4168. IEEE, 2009.

\bibitem[Fasel et~al.(2009)Fasel, Quinlan, and Stone]{fasel2009task}
Fasel, I.~R., Quinlan, M., and Stone, P.
\newblock {A Task Specification Language for Bootstrap Learning}.
\newblock In \emph{Proceedings of the AAAI Spring Symposium on Agents that
  Learn from Human Teachers}, pp.\  48--55, 2009.

\bibitem[Gori et~al.(2005)Gori, Monfardini, and Scarselli]{gnn_gori}
Gori, M., Monfardini, G., and Scarselli, F.
\newblock {A New Model for Learning in Graph Domains}.
\newblock In \emph{\longshort{\Proceedings \ofthe 2005 IEEE \International
  Joint \Conference on Neural Networks}{IJCNN}}, volume~2, pp.\  729--734,
  2005.

\bibitem[Hasanbeig et~al.(2018)Hasanbeig, Abate, and
  Kroening]{hasanbeig2018logically}
Hasanbeig, M., Abate, A., and Kroening, D.
\newblock {Logically-Constrained Reinforcement Learning}.
\newblock \emph{arXiv preprint arXiv:1801.08099}, 2018.

\bibitem[Hasanbeig et~al.(2020)Hasanbeig, Kroening, and
  Abate]{hasanbeig2020deep}
Hasanbeig, M., Kroening, D., and Abate, A.
\newblock {Deep Reinforcement Learning with Temporal Logics}.
\newblock In \emph{\longshort{\Proceedings \ofthe 18th \International
  \Conference on Formal Modeling and Analysis of Timed Systems}{FORMATS}}, pp.\
   1--22. Springer, 2020.

\bibitem[Hermann et~al.(2017)Hermann, Hill, Green, Wang, Faulkner, Soyer,
  Szepesvari, Czarnecki, Jaderberg, Teplyashin,
  et~al.]{DBLP:journals/corr/HermannHGWFSSCJ17}
Hermann, K.~M., Hill, F., Green, S., Wang, F., Faulkner, R., Soyer, H.,
  Szepesvari, D., Czarnecki, W.~M., Jaderberg, M., Teplyashin, D., et~al.
\newblock {Grounded Language Learning in a Simulated 3D World}.
\newblock \emph{arXiv preprint arXiv:1706.06551}, 2017.

\bibitem[Hill et~al.(2021)Hill, Tieleman, von Glehn, Wong, Merzic, and
  Clark]{hill2021grounded}
Hill, F., Tieleman, O., von Glehn, T., Wong, N., Merzic, H., and Clark, S.
\newblock {Grounded Language Learning Fast and Slow}.
\newblock In \emph{\longshort{\Proceedings \ofthe 9th \International
  \Conference on Learning Representations}{ICLR}}, 2021.

\bibitem[Jiang et~al.(2019)Jiang, Gu, Murphy, and
  Finn]{DBLP:conf/nips/JiangGMF19}
Jiang, Y., Gu, S., Murphy, K., and Finn, C.
\newblock {Language as an Abstraction for Hierarchical Deep Reinforcement
  Learning}.
\newblock In \emph{\longshort{\Proceedings \ofthe 32nd \Conference on Advances
  in Neural Information Processing Systems}{NeurIPS}}, volume~32, pp.\
  9414--9426, 2019.

\bibitem[Jiang et~al.(2020)Jiang, Bharadwaj, Wu, Shah, Topcu, and
  Stone]{jiang2020temporal}
Jiang, Y., Bharadwaj, S., Wu, B., Shah, R., Topcu, U., and Stone, P.
\newblock {Temporal-Logic-Based Reward Shaping for Continuing Learning Tasks}.
\newblock \emph{arXiv preprint arXiv:2007.01498}, 2020.

\bibitem[Jothimurugan et~al.(2019)Jothimurugan, Alur, and Bastani]{alur2019}
Jothimurugan, K., Alur, R., and Bastani, O.
\newblock {A Composable Specification Language for Reinforcement Learning
  Tasks}.
\newblock In \emph{\longshort{\Proceedings \ofthe 32nd \Conference on Advances
  in Neural Information Processing Systems}{NeurIPS}}, volume~32, pp.\
  13041--13051, 2019.

\bibitem[Kasenberg \& Scheutz(2017)Kasenberg and
  Scheutz]{DBLP:conf/cdc/KasenbergS17}
Kasenberg, D. and Scheutz, M.
\newblock Interpretable apprenticeship learning with temporal logic
  specifications.
\newblock In \emph{\longshort{\Proceedings \ofthe 56th IEEE Annual \Conference
  on Decision and Control}{CDC}}, pp.\  4914--4921. {IEEE}, 2017.

\bibitem[Kuo et~al.(2020)Kuo, Katz, and Barbu]{kuo2020encoding}
Kuo, Y.-L., Katz, B., and Barbu, A.
\newblock {Encoding Formulas as Deep Networks: Reinforcement Learning for
  Zero-Shot Execution of LTL Formulas}.
\newblock \emph{arXiv preprint arXiv:2006.01110}, 2020.

\bibitem[Kupferman \& Vardi(2001)Kupferman and Vardi]{kupferman2001model}
Kupferman, O. and Vardi, M.~Y.
\newblock Model checking of safety properties.
\newblock \emph{Formal Methods in System Design}, 19\penalty0 (3):\penalty0
  291--314, 2001.

\bibitem[Lake(2019)]{lake2019compositional}
Lake, B.~M.
\newblock {Compositional generalization through meta sequence-to-sequence
  learning}.
\newblock In \emph{\longshort{\Proceedings \ofthe 32nd \Conference on Advances
  in Neural Information Processing Systems}{NeurIPS}}, volume~32, 2019.

\bibitem[Leon et~al.(2020)Leon, Shanahan, and Belardinelli]{leon2020systematic}
Leon, B.~G., Shanahan, M., and Belardinelli, F.
\newblock {Systematic Generalisation through Task Temporal Logic and Deep
  Reinforcement Learning}.
\newblock \emph{arXiv preprint arXiv:2006.08767}, 2020.

\bibitem[{Li} et~al.(2017){Li}, {Vasile}, and {Belta}]{li2016temporal}
{Li}, X., {Vasile}, C., and {Belta}, C.
\newblock {Reinforcement Learning with Temporal Logic Rewards}.
\newblock In \emph{\longshort{\Proceedings \ofthe 2017 IEEE/RSJ \International
  \Conference on Intelligent Robots and Systems}{IROS}}, pp.\  3834--3839,
  2017.

\bibitem[Li et~al.(2018)Li, Ma, and Belta]{li2018policy}
Li, X., Ma, Y., and Belta, C.
\newblock {A Policy Search Method for Temporal Logic Specified Reinforcement
  Learning Tasks}.
\newblock In \emph{\longshort{\Proceedings \ofthe 2018 Annual American Control
  \Conference}{ACC}}, pp.\  240--245. IEEE, 2018.

\bibitem[Littman et~al.(2017)Littman, Topcu, Fu, Isbell, Wen, and
  MacGlashan]{littman2017environment}
Littman, M.~L., Topcu, U., Fu, J., Isbell, C., Wen, M., and MacGlashan, J.
\newblock {Environment-Independent Task Specifications via GLTL}.
\newblock \emph{arXiv preprint arXiv:1704.04341}, 2017.

\bibitem[Luketina et~al.(2019)Luketina, Nardelli, Farquhar, Foerster, Andreas,
  Grefenstette, Whiteson, and
  Rockt{\"a}schel]{DBLP:conf/ijcai/LuketinaNFFAGWR19}
Luketina, J., Nardelli, N., Farquhar, G., Foerster, J., Andreas, J.,
  Grefenstette, E., Whiteson, S., and Rockt{\"a}schel, T.
\newblock {A Survey of Reinforcement Learning Informed by Natural Language}.
\newblock In \emph{\longshort{\Proceedings \ofthe 28th \International Joint
  \Conference on \AI{}}{IJCAI}}, volume~57, pp.\  6309--6317, 2019.

\bibitem[McCarthy et~al.(1960)]{mccarthy1960programs}
McCarthy, J. et~al.
\newblock \emph{Programs with common sense}.
\newblock RLE and MIT computation center, 1960.

\bibitem[Moarref \& Kress{-}Gazit(2017)Moarref and
  Kress{-}Gazit]{DBLP:conf/mrs/MoarrefK17}
Moarref, S. and Kress{-}Gazit, H.
\newblock Decentralized control of robotic swarms from high-level temporal
  logic specifications.
\newblock In \emph{\longshort{\Proceedings \ofthe 2017 \International Symposium
  on Multi-Robot and Multi-Agent Systems}{MRS}}, pp.\  17--23. {IEEE}, 2017.

\bibitem[Moarref \& Kress{-}Gazit(2020)Moarref and
  Kress{-}Gazit]{DBLP:journals/arobots/MoarrefK20}
Moarref, S. and Kress{-}Gazit, H.
\newblock Automated synthesis of decentralized controllers for robot swarmsfrom
  high-level temporal logic specifications.
\newblock \emph{Autonomous Robots}, 44\penalty0 (3-4):\penalty0 585--600, 2020.

\bibitem[Oh et~al.(2017)Oh, Singh, Lee, and Kohli]{DBLP:conf/icml/OhSLK17}
Oh, J., Singh, S., Lee, H., and Kohli, P.
\newblock {Zero-Shot Task Generalization with Multi-Task Deep Reinforcement
  Learning}.
\newblock In \emph{\longshort{\Proceedings \ofthe 34th \International
  \Conference on Machine Learning}{ICML}}, pp.\  2661--2670. PMLR, 2017.

\bibitem[Pnueli(1977)]{DBLP:conf/focs/Pnueli77}
Pnueli, A.
\newblock {The Temporal Logic of Programs}.
\newblock In \emph{\longshort{\Proceedings \ofthe 18th IEEE Symposium on
  Foundations of Computer Science}{FOCS}}, pp.\  46--57. IEEE, 1977.

\bibitem[Raman et~al.(2012)Raman, Finucane, and
  Kress{-}Gazit]{DBLP:conf/iros/RamanFK12}
Raman, V., Finucane, C., and Kress{-}Gazit, H.
\newblock Temporal logic robot mission planning for slow and fast actions.
\newblock In \emph{\longshort{\Proceedings \ofthe 2012 IEEE/RSJ \International
  \Conference on Intelligent Robots and Systems}{IROS}}, pp.\  251--256.
  {IEEE}, 2012.

\bibitem[Ray et~al.(2019)Ray, Achiam, and Amodei]{safetygym}
Ray, A., Achiam, J., and Amodei, D.
\newblock {Benchmarking Safe Exploration in Deep Reinforcement Learning}.
\newblock \emph{arXiv preprint arXiv:1910.01708}, 2019.

\bibitem[Scarselli et~al.(2008)Scarselli, Gori, Tsoi, Hagenbuchner, and
  Monfardini]{scarselli2008graph}
Scarselli, F., Gori, M., Tsoi, A.~C., Hagenbuchner, M., and Monfardini, G.
\newblock The graph neural network model.
\newblock \emph{IEEE transactions on neural networks}, 20\penalty0
  (1):\penalty0 61--80, 2008.

\bibitem[Schlichtkrull et~al.(2018)Schlichtkrull, Kipf, Bloem, Van Den~Berg,
  Titov, and Welling]{DBLP:conf/esws/SchlichtkrullKB18}
Schlichtkrull, M., Kipf, T.~N., Bloem, P., Van Den~Berg, R., Titov, I., and
  Welling, M.
\newblock {Modeling Relational Data with Graph Convolutional Networks}.
\newblock In \emph{\longshort{\Proceedings \ofthe 15th European Semantic Web
  \Conference}{ESWC}}, pp.\  593--607. Springer, 2018.

\bibitem[Schulman et~al.(2017)Schulman, Wolski, Dhariwal, Radford, and
  Klimov]{DBLP:journals/corr/SchulmanWDRK17}
Schulman, J., Wolski, F., Dhariwal, P., Radford, A., and Klimov, O.
\newblock {Proximal Policy Optimization Algorithms}.
\newblock \emph{arXiv preprint arXiv:1707.06347}, 2017.

\bibitem[Selsam et~al.(2018)Selsam, Lamm, B{\"u}nz, Liang, de~Moura, and
  Dill]{selsam2018learning}
Selsam, D., Lamm, M., B{\"u}nz, B., Liang, P., de~Moura, L., and Dill, D.~L.
\newblock Learning a sat solver from single-bit supervision.
\newblock \emph{arXiv preprint arXiv:1802.03685}, 2018.

\bibitem[Shah et~al.(2018)Shah, Kamath, Shah, and Li]{DBLP:conf/nips/ShahKSL18}
Shah, A., Kamath, P., Shah, J.~A., and Li, S.
\newblock {Bayesian Inference of Temporal Task Specifications from
  Demonstrations}.
\newblock In \emph{\longshort{\Proceedings \ofthe 31st \Conference on Advances
  in Neural Information Processing Systems}{NeurIPS}}, pp.\  3808--3817, 2018.

\bibitem[Shah et~al.(2020)Shah, Li, and Shah]{DBLP:journals/ral/ShahLS20}
Shah, A., Li, S., and Shah, J.
\newblock {Planning With Uncertain Specifications (PUnS)}.
\newblock \emph{{IEEE} Robotics Autom. Lett.}, 5\penalty0 (2):\penalty0
  3414--3421, 2020.

\bibitem[Si et~al.(2018)Si, Dai, Raghothaman, Naik, and
  Song]{DBLP:conf/nips/SiDRNS18}
Si, X., Dai, H., Raghothaman, M., Naik, M., and Song, L.
\newblock {Learning Loop Invariants for Program Verification}.
\newblock In \emph{\longshort{\Proceedings \ofthe 31st \Conference on Advances
  in Neural Information Processing Systems}{NeurIPS}}, 2018.

\bibitem[Sohn et~al.(2018)Sohn, Oh, and Lee]{DBLP:conf/nips/SohnOL18}
Sohn, S., Oh, J., and Lee, H.
\newblock {Hierarchical Reinforcement Learning for Zero-shot Generalization
  with Subtask Dependencies}.
\newblock In \emph{\longshort{\Proceedings \ofthe 31st \Conference on Advances
  in Neural Information Processing Systems}{NeurIPS}}, volume~31, pp.\
  7156--7166, 2018.

\bibitem[Sun et~al.(2019)Sun, Wu, and Lim]{DBLP:conf/iclr/SunWL20}
Sun, S.-H., Wu, T.-L., and Lim, J.~J.
\newblock {Program Guided Agent}.
\newblock In \emph{\longshort{\Proceedings \ofthe 7th \International
  \Conference on Learning Representations}{ICLR}}, 2019.

\bibitem[{Toro Icarte} et~al.(2018{\natexlab{a}}){Toro Icarte}, Klassen,
  Valenzano, and McIlraith]{DBLP:conf/atal/IcarteKVM18}
{Toro Icarte}, R., Klassen, T.~Q., Valenzano, R., and McIlraith, S.~A.
\newblock {Teaching Multiple Tasks to an RL Agent using LTL}.
\newblock In \emph{\longshort{\Proceedings \ofthe 17th \International
  \Conference on Autonomous Agents and Multiagent Systems}{AAMAS}}, pp.\
  452--461, 2018{\natexlab{a}}.

\bibitem[{Toro Icarte} et~al.(2018{\natexlab{b}}){Toro Icarte}, Klassen,
  Valenzano, and McIlraith]{DBLP:conf/icml/IcarteKVM18}
{Toro Icarte}, R., Klassen, T.~Q., Valenzano, R., and McIlraith, S.~A.
\newblock {Using Reward Machines for High-Level Task Specification and
  Decomposition in Reinforcement Learning}.
\newblock In \emph{\longshort{\Proceedings \ofthe 35th \International
  \Conference on Machine Learning}{ICML}}, pp.\  2107--2116,
  2018{\natexlab{b}}.

\bibitem[{Toro Icarte} et~al.(2020){Toro Icarte}, Klassen, Valenzano, and
  McIlraith]{DBLP:journals/corr/abs-2010-03950}
{Toro Icarte}, R., Klassen, T.~Q., Valenzano, R., and McIlraith, S.~A.
\newblock {Reward Machines: Exploiting Reward Function Structure in
  Reinforcement Learning}.
\newblock \emph{arXiv preprint arXiv:2010.03950}, 2020.

\bibitem[Vaezipoor et~al.(2020)Vaezipoor, Lederman, Wu, Maddison, Grosse, Lee,
  Seshia, and Bacchus]{vaezipoor2020learning}
Vaezipoor, P., Lederman, G., Wu, Y., Maddison, C.~J., Grosse, R., Lee, E.,
  Seshia, S.~A., and Bacchus, F.
\newblock Learning branching heuristics for propositional model counting.
\newblock \emph{arXiv preprint arXiv:2007.03204}, 2020.

\bibitem[Wang et~al.(2020)Wang, Ross, Katz, and Barbu]{wang2020learning}
Wang, C., Ross, C., Katz, B., and Barbu, A.
\newblock {Learning a Natural-Language to LTL Executable Semantic Parser for
  Grounded Robotics}.
\newblock \emph{arXiv preprint arXiv:2008.03277}, 2020.

\bibitem[Xu et~al.(2018)Xu, Nair, Zhu, Gao, Garg, Fei-Fei, and
  Savarese]{DBLP:conf/icra/XuNZGGFS18}
Xu, D., Nair, S., Zhu, Y., Gao, J., Garg, A., Fei-Fei, L., and Savarese, S.
\newblock {Neural Task Programming: Learning to Generalize Across Hierarchical
  Tasks}.
\newblock In \emph{\longshort{\Proceedings \ofthe 2018 IEEE \International
  \Conference on Robotics and Automation}{ICRA}}, pp.\  3795--3802. IEEE, 2018.

\bibitem[Xu \& Topcu(2019)Xu and Topcu]{DBLP:conf/ijcai/0005T19}
Xu, Z. and Topcu, U.
\newblock {Transfer of Temporal Logic Formulas in Reinforcement Learning}.
\newblock In \emph{\longshort{\Proceedings \ofthe 28th \International Joint
  \Conference on \AI{}}{IJCAI}}, pp.\  4010--4018, 2019.

\bibitem[Yu et~al.(2018)Yu, Zhang, and Xu]{DBLP:conf/iclr/YuZX18}
Yu, H., Zhang, H., and Xu, W.
\newblock {Interactive Grounded Language Acquisition and Generalization in a 2D
  World}.
\newblock In \emph{\longshort{\Proceedings \ofthe 6th \International
  \Conference on Learning Representations}{ICLR}}, 2018.

\bibitem[Yuan et~al.(2019)Yuan, Hasanbeig, Abate, and Kroening]{deeplcrl}
Yuan, L.~Z., Hasanbeig, M., Abate, A., and Kroening, D.
\newblock {Modular Deep Reinforcement Learning with Temporal Logic
  Specifications}.
\newblock \emph{arXiv preprint arXiv:1909.11591}, 2019.

\end{thebibliography}
\bibliographystyle{icml2021}

\
\newpage

\appendix
\twocolumn[
\icmltitle{LTL2Action: Generalizing LTL Instructions for Multi-Task RL (Appendix)}
]

\section{Proof of Theorem  3.1}
\label{app:theorem}

\paragraph{Theorem  3.1.}
\begin{itpars}
    Let $\mathcal{M}_\Phi = \tuple{S',T',A,\prob',R',\gamma,\mu'}$ be a Taskable MDP constructed from an MDP without a reward function $\mathcal{M}_e = \tuple{S,T,A,\prob,\gamma,\mu}$, a finite set of propositional symbols $\mathcal{P}$, a labelling function $L: S \times A \rightarrow 2^\mathcal{P}$, a finite set of LTL formulas $\Phi$, and a probability distribution $\tau$ over $\Phi$ according to Definition \ref{def:taskableMDP}. Then, an optimal stationary policy $\pi_\Phi^*(a|s,\varphi)$ for $\mathcal{M}_\Phi$ achieves the same expected discounted return as an optimal non-stationary policy $\pi_\varphi^*(a_t|s, a_1, ..., s_t,\varphi)$ for $\mathcal{M}_e$ w.r.t.\ $R_\varphi$ for all $s \in S$ and $\varphi \in \Phi$. 
    
\end{itpars}

To prove Theorem~\ref{theo:cross-product}, we use a theorem from \citet{DBLP:journals/ai/BacchusK00}, which shows the correctness of LTL progression:
\begin{theorem}
    \label{the:progre}
    Given any LTL formula $\varphi$ and an infinite sequence of truth assignments $\sigma=\langle \sigma_i, \sigma_{i+1}, \sigma_{i+2}, \ldots \rangle$ for the variables in $\mathcal{P}$, $\tuple{\sigma,i}\models \varphi$ iff $\tuple{\sigma,i+1}\models\mprog(\sigma_i,\varphi)$.
\end{theorem}
\begin{proof}
Using induction and Theorem~\ref{the:progre}, we can prove that the reward given by $R_\varphi$ and $r'$ is identical (at every time step) for any LTL formula $\varphi \in \Phi$, initial state $s_1 \in S$, and trace $s_1, a_1, ..., s_t,a_t$. Now, let's consider any state $s \in S$ and task $\varphi \in \Phi$. Given any optimal policy $\pi_\Phi^*(a|s,\varphi)$ for $\mathcal{M}_\Phi$, we can construct a policy $\pi_\varphi(a_t|s, a_1, ..., s_t,\varphi)$ for $\mathcal{M}_e$ that mimics the actions selection of $\pi_\Phi^*(a|s,\varphi)$ step by step. Hence, as the probability of reaching state $s'$ given state $s$ and action $a$ is the same for $\mathcal{M}_\Phi$ and $\mathcal{M}_e$, both policies will induce the same probability distribution over traces and, as the reward functions are equivalent, both policies $\pi_\Phi^*$ and $\pi_\varphi$ achieve the same expected discounted return. Finally, if we now have an optimal policy $\pi_\varphi^*(a_t|s, a_1, ..., s_t,\varphi)$ for $\mathcal{M}_e$, we can construct a non-stationary policy $\pi_\Phi(a_t|\tuple{s,\varphi}, a_1, ..., \tuple{s_t,\varphi_t})$ for $\mathcal{M}_\Phi$ that mimics the actions selection of $\pi_\varphi^*$ step by step. Following the same argument as before, we can see that $\pi_\varphi^*$ and $\pi_\Phi$ achieve the same expected discounted return. Since $\mathcal{M}_\Phi$ is an MDP, we know that there exist a stationary policy $\pi_\Phi'(a|s,\varphi)$ that achieves at least as much return as any non-stationary policy $\pi_\Phi(a_t|\tuple{s,\varphi}, a_1, ..., \tuple{s_t,\varphi_t})$. Therefore, we showed that optimal policies for $\mathcal{M}_\Phi$ are as good as optimal policies for $\mathcal{M}_e$ w.r.t.\ any $R_\varphi$ (and vice versa).
\end{proof}

\section{Experimental Details}

\begin{table*}[t]
\small
\centering
\caption{
PPO hyperparameters for \letterWorld. The same set of hyperparameters were used for both Avoidance and Partially-Ordered tasks.}
\label{table:ppo_hyper_letter_world}

\begin{tabular}{*{9}c}

\cmidrule[\heavyrulewidth]{1-9}
& \multicolumn{1}{c}{\gnnpp} &
  \multicolumn{1}{c}{\grupp} &
  \multicolumn{1}{c}{\gnnp} &
  \multicolumn{1}{c}{\grup} &
  \multicolumn{1}{c}{\lstmp} &
  \multicolumn{1}{c}{Myopic} &
  \multicolumn{1}{c}{GRU} &
  \multicolumn{1}{c}{No LTL}
   \\ 

\cmidrule[\heavyrulewidth]{1-9}

\multicolumn{1}{l}{Env. steps per update} & \multicolumn{8}{c}{$\xleftarrow{\hspace*{0.5\columnwidth}} 2,048 \xrightarrow{\hspace*{0.5\columnwidth}}$}  \\ 
\multicolumn{1}{l}{Number of epochs} & 4 & 8 & 4 & 8 & 8 & 8 & 4 & 4  \\ 
\multicolumn{1}{l}{Minibatch Size} & 256 & 256 & 256 & 256 & 256 & 256 & 1,024 & 1,024 \\ 
\multicolumn{1}{l}{Discount factor ($\gamma$)} & \multicolumn{8}{c}{$\xleftarrow{\hspace*{0.5\columnwidth}} 0.94 \xrightarrow{\hspace*{0.5\columnwidth}}$}   \\ 
\multicolumn{1}{l}{Learning rate} & $3\times10^{-4}$ & $3\times10^{-4}$ & $3\times10^{-4}$ & $3\times10^{-4}$ & $3\times10^{-4}$ & $10^{-4}$ & $3\times10^{-4}$ & $3\times10^{-4}$ \\
\multicolumn{1}{l}{GAE-$\lambda$} & \multicolumn{8}{c}{$\xleftarrow{\hspace*{0.5\columnwidth}} 0.95 \xrightarrow{\hspace*{0.5\columnwidth}}$} \\
\multicolumn{1}{l}{Entropy coefficient} & \multicolumn{8}{c}{$\xleftarrow{\hspace*{0.5\columnwidth}} 0.01 \xrightarrow{\hspace*{0.5\columnwidth}}$} \\
\multicolumn{1}{l}{Value loss coefficient} & \multicolumn{8}{c}{$\xleftarrow{\hspace*{0.5\columnwidth}} 0.5 \xrightarrow{\hspace*{0.5\columnwidth}}$} \\
\multicolumn{1}{l}{Gradient Clipping} & \multicolumn{8}{c}{$\xleftarrow{\hspace*{0.5\columnwidth}} 0.5 \xrightarrow{\hspace*{0.5\columnwidth}}$} \\
\multicolumn{1}{l}{PPO Clipping ($\varepsilon$)} & \multicolumn{8}{c}{$\xleftarrow{\hspace*{0.5\columnwidth}} 0.2 \xrightarrow{\hspace*{0.5\columnwidth}}$} \\

\bottomrule
\end{tabular}
\vspace{-2mm}
\end{table*}

In this section we provide some details on the task generation process as well as the hyperparameters used for our model training.
\subsection{LTL Task Generation}
\label{app:ltl_tasks}
Recall that we consider two task spaces: Partially-Ordered Tasks and Avoidance Tasks. The random generation of tasks is best described recursively with production rules of a context-free grammar. 
\\

\begin{tightcenter}
{\textbf{\underline{Partially-Ordered Tasks}}}
\vspace{-2mm}
\end{tightcenter}
\begin{align*}
\mathbf{formula} &\rightarrow \mathbf{sequence} \wedge \mathbf{formula} \mid \mathbf{sequence} \\
\mathbf{sequence} &\rightarrow {\Diamond}(\mathbf{term} \wedge \mathbf{sequence}) \mid \Diamond \mathbf{term} \\
\mathbf{term} &\rightarrow \mathsf{prop} \mid \mathsf{prop} \vee \mathsf{prop}
\end{align*}

In the above description, $\Diamond, \wedge, \vee$ are the \emph{eventually}, \emph{and}, \emph{or} LTL operators, respectively and $\mathsf{prop}$ refers to any propositional variable. 

Intuitively, Partially-Ordered Tasks presents $k$ sequences of propositions which can be solved simultaneously. A trace is successful if and only if for every one of the $k$ sequences, all the propositions in that sequence occur at some point in the trace (in the order of the sequence).
Note that Partially-Ordered Tasks can never be falsified. However, most tasks are computationally intractable to solve in as few steps as possible due to the exponential number of possible solutions which must be considered. An example formula that is a conjunction of 2 sequences, each of depth 2 is: 
$$\Diamond ((A \vee B) \wedge \Diamond C) \wedge \Diamond (C \wedge \Diamond D))$$

In our Letter World experiments, the number of conjuncts was randomly sampled between 1 and 4, and the depth of each sequence was randomly sampled between 1 and 5. Each ``term" had a 0.25 probability of being a disjunction of two propositions, and a 0.75 probability of being a single proposition.

To evaluate generalization to larger formulas, we considered (separately) increasing the depth of sequences and increasing the number of conjuncts. For increased depth tasks, the depth was 15 and the number of conjuncts was randomly sampled between 2 and 4. For increased number of conjuncts, the depth was randomly sampled between 3 and 5, and the number of conjuncts was 12. 

\begin{tightcenter}
\textbf{\underline{Avoidance Tasks}}
\vspace{-2mm}
\end{tightcenter}
\begin{align*}
\mathbf{formula} &\rightarrow \mathbf{sequence} \wedge \mathbf{formula} \mid \mathbf{sequence} \\
\mathbf{sequence} &\rightarrow \neg \mathsf{prop} \ltluntil (\mathsf{prop} \wedge \mathbf{sequence}) \mid \neg \mathsf{prop} \ltluntil \mathsf{prop}
\end{align*}

Here, the $\neg, \ltluntil$ symbols are the \emph{not}, \emph{until} LTL operators, respectively.
Similar to Partially-Ordered Tasks, several parallel sequences of propositions must be satisfied. However, this task space introduces the added challenge of propositions which must be avoided. The propositions to be avoided change as different parts of the task are solved.
An example formula that is a conjunction of two sequences, each of depth two is: 
$$(\neg A \ltluntil (K \wedge (\neg H \ltluntil J))) \wedge (\neg G \ltluntil (L \wedge (\neg F \ltluntil I)))$$

In order to guarantee that every formula can be solved, we do not allow the same proposition to appear twice in the same formula (avoiding conflicts such as $(\neg A \ltluntil A)$, which cannot be satisfied). In the Letter World, the number of conjuncts was randomly sampled between 1 and 2 and the depth of each sequence was randomly sampled between 1 and 3. For generalization to larger formulas, we considered depth 6 formulas with 1 conjunct (increased depth), as well as depth 2 formulas with 3 conjuncts (increased conjuncts). For the safety gym environment, we considered 1 conjunct, and randomly sampled the depth between 1 and 2 (longer tasks suffered from sparse reward, which is not the focus on this work). 

\subsection{Network Architectures}
\label{app:networks}

As mentioned in Section \ref{sec:arch}, we used PPO\footnote{We used \texttt{torch-ac}'s implementation of PPO (\url{https://github.com/lcswillems/torch-ac}).} as the RL method for our experiments. We used the same actor (3 fully-connected layers with [64, 64, 64] units and ReLU activations) and critic (3 fully-connected layers with [64, 64, 1] units and Tanh activation) model for \letterWorld and \zoneEnv. In \simLtl (pretraining), we used a single layer actor and critic with no hidden layers. This was to encourage the LTL module to learn a self-sufficient encoding (as the actor and critics cannot be transferred to downstream tasks).
For discrete action space environments, the actor's output was passed through a logit layer before softmax. For the continuous case we assumed a Gaussian action distribution and parameterized its mean and standard deviation by sending the actor's output to two separate linear layers. 

The Env Module is determined by the observation space of the underlying environment: in \letterWorld we used a 3-layer convolutional neural network with 16, 32 and 64 channels, kernel size of $2\times2$ and stride of 1 and in \zoneEnv we used a 2-layer fully-connected network with [128, 128] units and ReLU activations. Naturally, there is no Env Module for \simLtl.

For LTL Module, we tested the following architectures with roughly the same number of parameters ($\sim10^4$) to encode LTL formulas:
\begin{itemize}[topsep=0pt,itemsep=2pt,partopsep=0pt, parsep=0pt]

    \item \emph{Graph Neural Networks (GNN):} The R-GCN architecture of Section~\ref{sec:arch} with $T=8$ message passing steps and 32-dimensional node embeddings, i.e., $\boldsymbol{x}^{(t)}_v \in \mathbb{R}^{32}$. To reduce the number of trainable parameters we share the weight matrix across iterations for each edge type: $W_r = W_r^{(t)}$ ($0\le t \le T)$. We observed better expressibility by concatenating the embedding of a node at iteration $t$ with its one-hot encoding before passing it to the neighboring nodes for aggregation: $(\boldsymbol{x}^{(t)}_u, \boldsymbol{x}^{(0)}_u)$, thus $W_r \in \mathbb{R}^{(32+32)\times32}$. We used Tanh as the element-wise activation of Equation \ref{eq:rgcn}. 
    
    Note that the embedding does not consider information from nodes more than $T$ edges away from the root. However, this did not appear to be an issue in our experiments using $T = 8$, despite encountering formulas with ASTs larger than $8$. One mitigating factor is that LTL progression generally reduces the size of formulas as parts of the task are solved, and the information most immediately relevant to the task tends to lie closer to the root.
    \item \emph{Gated Recurrent Units (GRU):} A 2-layer bidirectional GRU with a 16-dimensional hidden layer. 
    \item \emph{Long Short-Term Memory (LSTM):} A 2-layer bidirectional LSTM with a 16-dimensional hidden layer. 
\end{itemize}

\subsection{PPO Hyperparameters}
\label{app:hyperparams}

All experiments were conducted on a compute cluster using 1 GPU and 16 CPU cores per run. The hyperparameters used for PPO for each baseline are displayed in Table~\ref{table:ppo_hyper_letter_world} for the \letterWorld, Table~\ref{table:ppo_hyper_safety} for the \zoneEnv, and Table~\ref{table:ppo_hyper_bootcamp} for the \simLtl (pretraining). Using a GNN architecture, training completed in approximately 26 hours in \letterWorld and 24 hours in the \zoneEnv (both for 20 million frames). Using a GRU to encode formulas was usually 2-3$\times$ more wall-clock efficient compared to GNN.

Note that all baselines which treat the problem as partially observable use an additional recurrent layer after the Env Model (i.e., \emph{GRU} and \emph{No LTL}). As backpropagation through all timesteps is computationally expensive, we backpropagate gradients only through the last 4 timesteps. We did not observe better performance by increasing the number of backpropagation steps.

\begin{table}[t]
\small
\centering
\caption{PPO hyperparameters for \zoneEnv. Only Avoidance tasks were considered on this environment.}
\label{table:ppo_hyper_safety}
\vspace{2mm}
\begin{tabular}{*{4}c}

\cmidrule[\heavyrulewidth]{1-4}
& \multicolumn{1}{c}{\gnnpp} &
  \multicolumn{1}{c}{\gnnp} &
  \multicolumn{1}{c}{Myopic}
   \\ 

\cmidrule[\heavyrulewidth]{1-4}

\multicolumn{1}{l}{Env. steps per update} & 65,536 & 65,536 & 65,536  \\ 
\multicolumn{1}{l}{Number of epochs} & 10 & 10 & 10 \\ 
\multicolumn{1}{l}{Minibatch size} & 2,048 & 2,048 & 1,024 \\ 
\multicolumn{1}{l}{Discount factor ($\gamma$)} &  0.998 &  0.998 &  0.998   \\ 
\multicolumn{1}{l}{Learning rate} & $3\times10^{-4}$ & $3\times10^{-4}$ & $3\times10^{-4}$ \\
\multicolumn{1}{l}{GAE-$\lambda$} & 0.95 & 0.95 & 0.95 \\
\multicolumn{1}{l}{Entropy coefficient} & 0.003 & 0.003 & 0.003 \\
\multicolumn{1}{l}{Value loss coefficient} & 0.5 & 0.5 & 0.5 \\
\multicolumn{1}{l}{Gradient Clipping} & 0.5 & 0.5 & 0.5  \\
\multicolumn{1}{l}{PPO Clipping ($\varepsilon$)} & 0.2 & 0.2 & 0.2 \\

\bottomrule
\end{tabular}
\vspace{-3mm}
\end{table}

\begin{table}[t]
\small
\centering
\caption{PPO hyperparameters for \simLtl (pretraining).}
\label{table:ppo_hyper_bootcamp}
\vspace{2mm}
\begin{tabular}{*{3}c}

\cmidrule[\heavyrulewidth]{1-3}
& \multicolumn{1}{c}{\gnnp} &
  \multicolumn{1}{c}{\grup}
   \\ 

\cmidrule[\heavyrulewidth]{1-3}
   
\multicolumn{3}{c}{(a) Avoidance Tasks} \\ \cmidrule(lr){1-3} 

\multicolumn{1}{l}{Env. steps per update} & 8,192 & 8,192  \\ 
\multicolumn{1}{l}{Number of epochs} & 2 & 2 \\ 
\multicolumn{1}{l}{Minibatch size} & 1,024 & 1,024 \\ 
\multicolumn{1}{l}{Discount factor ($\gamma$)} & 0.9 & 0.9  \\ 
\multicolumn{1}{l}{Learning rate} & $10^{-3}$ & $10^{-3}$ \\
\multicolumn{1}{l}{GAE-$\lambda$} & 0.5 & 0.5 \\
\multicolumn{1}{l}{Entropy coefficient} & 0.01 & 0.01 \\
\multicolumn{1}{l}{Value loss coefficient} & 0.5 & 0.5 \\
\multicolumn{1}{l}{Gradient Clipping} & 0.5 & 0.5 \\
\multicolumn{1}{l}{PPO Clipping ($\varepsilon$)} & 0.1 & 0.1 \\

\cmidrule(lr){1-3} 
\multicolumn{3}{c}{(b) Partially-Ordered Tasks} \\ \cmidrule(lr){1-3} 

\multicolumn{1}{l}{Env. steps per update} & 8,192 & 8,192  \\ 
\multicolumn{1}{l}{Number of epochs} & 2 & 4 \\ 
\multicolumn{1}{l}{Minibatch size} & 1,024 & 1,024\\ 
\multicolumn{1}{l}{Discount factor ($\gamma$)} & 0.9 & 0.9  \\ 
\multicolumn{1}{l}{Learning rate} & $10^{-3}$ & $3\times10^{-3}$ \\
\multicolumn{1}{l}{GAE-$\lambda$} & 0.5 & 0.95 \\
\multicolumn{1}{l}{Entropy coefficient} & 0.01 & 0.01 \\
\multicolumn{1}{l}{Value loss coefficient} & 0.5 & 0.5 \\
\multicolumn{1}{l}{Gradient Clipping} & 0.5 & 0.5 \\
\multicolumn{1}{l}{PPO Clipping ($\varepsilon$)} & 0.1 & 0.2 \\

\bottomrule
\end{tabular}
\end{table}

\section{Additional Results}

\subsection{Generalization to Unseen Objects}
\label{sec:object-gen}
While the main focus of our work was generalization to new instructions, an important related problem is generalization to unseen \emph{objects} \cite{hill2021grounded, leon2020systematic}. We conduct a simple experiment in \letterWorld to test object generalization in our framework by evaluating on unseen letters/propositions. While our framework normally uses one-hot encodings for propositions, such an approach is not conducive to generalization to new propositions. Here, we instead encode each letter as a random (but fixed), normalized vector in a low-dimensional space $\mathbb{R}^d$ (we use $d=3$). Importantly, the same proposition is encoded in the same way in both the grid and in LTL formulas. We consider Avoidance tasks with a depth of 2 and train our agent on formulas over 12 fixed letters. We then evaluate this agent (over 5 seeds and 1000 episodes per seed) on formulas with the same structure, but over (a) 6 previously seen and 6 unseen letters, and
(b) 12 unseen letters. 

Results are displayed in Table~\ref{table:obj_generalization}. Compared to a random action-selection baseline, our framework generalizes well to new tasks over unseen letters.
\begin{table}[t]
\small
\centering
\caption{
RL agents are trained on LTL tasks over 12 letters, and are evaluated on tasks over some unseen letters. In each entry, we report the mean return over 5 seeds and 1000 episodes per seed, with 90\% confidence error.}
\label{table:obj_generalization}
\vspace{2mm}

\begin{tabular}{lcc}
\cmidrule[\heavyrulewidth]{2-3}
\multicolumn{1}{c}{} & \multicolumn{2}{c}{\% Unseen Letters} \\
\cmidrule[\heavyrulewidth]{2-3} 
\multicolumn{1}{c}{} & 50\%                 & 100\%                \\
\cmidrule{2-3} 
Ours                  & $0.667 \pm 0.015$       & $0.607 \pm 0.016$       \\
Random                & $-0.374 \pm 0.021$      & $-0.374 \pm 0.021$  \\    
\cmidrule[\heavyrulewidth]{1-3}
\end{tabular}
\end{table}

\subsection{Pretraining Learning Curves}

\begin{figure*}[t]
    \centering
    \begin{tikzpicture}
    \small
    
    \node at (2.5,0) {\includegraphics[width=0.95\columnwidth]{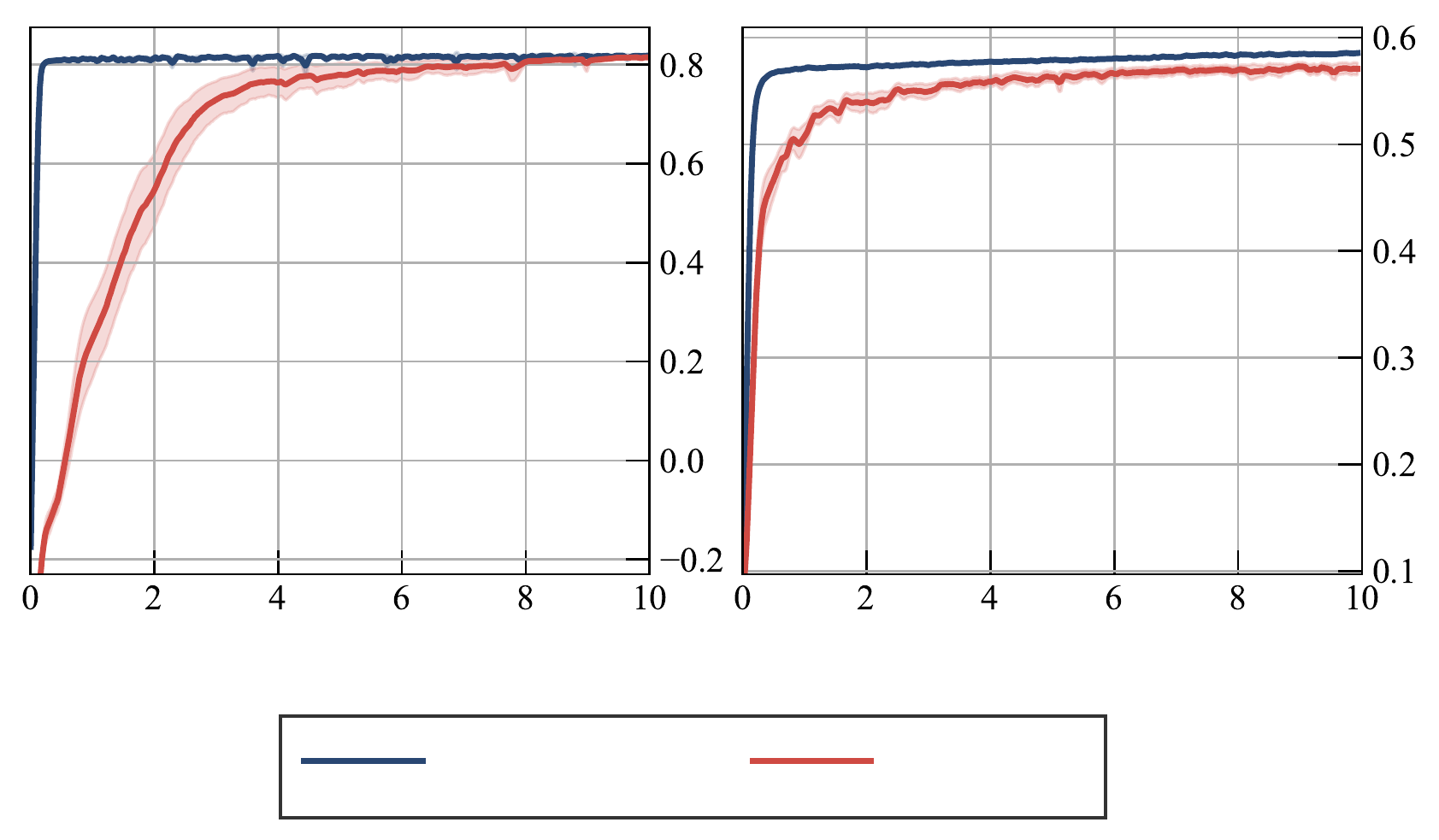}};
    \draw (0.4,2.4) node {Avoidance Tasks}; 
    \draw (4.35,2.4) node {Partially-Ordered Tasks}; 
    \draw (-1.5,0.7) node[rotate=90] {Discounted return};
    \draw (2.3,-1.3) node {Frames (millions)}; 
    
    \draw (1.55,-1.89) node {\gnnp};
    \draw (3.95,-1.89) node {\grup}; 
    
    \end{tikzpicture}
  \caption{The learning curves of the GNN and GRU (both with progression) in the \simLtl (pretraining) environment. Given random formulas, the task is to choose propositions which satisfy it in as few steps as possible.
  We report \emph{discounted return} over the duration of training (averaged over 30 seeds, with 90\% confidence intervals).
  } 
  \label{fig:pretraining_curves}
  \vspace{150mm}
\end{figure*}
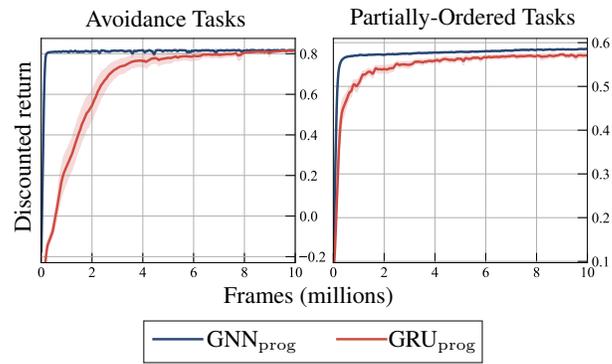

In Figure~\ref{fig:pretraining_curves}, we report the learning curves of {\gnnp} and {\grup} on the \simLtl environment. Note that this is not meant to be an evaluation benchmark and is only used for pretraining the LTL module in our other experiments (see the main text, Figure~\ref{fig:transfer_graph}). We observe, however, that the GNN is able to learn significantly faster than the GRU. Note that the Partially-Ordered tasks still remain extremely challenging to solve optimally, even in this abstracted environment.

    
    
    
    



\end{document}